\documentclass[sigconf]{acmart}

\usepackage{algorithm}
\usepackage{algpseudocode}
\usepackage{xspace}
\usepackage{amsmath}
\usepackage{subcaption}
\usepackage{multirow}

\newcommand{\name}{RapidGNN\xspace}

\acmConference[SC2025 (Sustainable Supercomputing Workshop)]{ACM/IEEE Conference}{November 2025}{St. Louis, MO}

\begin{document}

\title{RapidGNN: Energy and Communication-Efficient Distributed Training on Large-Scale Graph Neural Networks}

\author{Arefin Niam}
\email{aniam42@tntech.edu}
\affiliation{%
  \institution{Tennessee Technological University}
  \city{Cookeville}
  \state{Tennessee}
  \country{USA}
}

\author{Tevfik Kosar}
\email{tkosar@buffalo.edu}
\affiliation{%
  \institution{University at Buffalo}
  \city{Buffalo}
  \state{New York}
  \country{USA}
}

\author{M S Q Zulkar Nine}
\email{mnine@tntech.edu}
\affiliation{%
  \institution{Tennessee Technological University}
  \city{Cookeville}
  \state{Tennessee}
  \country{USA}
}

\begin{abstract}
    Graph Neural Networks (GNNs) have become popular across a diverse set of tasks in exploring structural relationships between entities. However, due to the highly connected structure of the datasets, distributed training of GNNs on large-scale graphs poses significant challenges. Traditional sampling-based approaches mitigate the computational loads, yet the communication overhead remains a challenge. This paper presents \name,  a distributed GNN training framework with deterministic sampling-based scheduling to enable efficient cache construction and prefetching of remote features. Evaluation on benchmark graph datasets demonstrates \name's effectiveness across different scales and topologies. \name improves end-to-end training throughput by \textbf{2.46$\times$} to \textbf{3.00$\times$} on average over baseline methods across the benchmark datasets, while cutting remote feature fetches by over \textbf{9.70$\times$} to \textbf{15.39$\times$}. \name further demonstrates near-linear scalability with an increasing number of computing units efficiently. Furthermore, it achieves increased energy efficiency over the baseline methods for both CPU and GPU by \textbf{44\%} and \textbf{32\%}, respectively.
\end{abstract}


\maketitle

\section{Introduction}
Recently, the application of GNNs has brought breakthrough results in a number of scientific and engineering problems (e.g., molecular property prediction and drug discovery~\cite{gilmer2017neural,xiong2021graph,li2022multiphysical}, protein structure prediction~\cite{jumper2021highly,jha2022prediction,reau2023deeprank}, material science and crystal structure prediction~\cite{reiser2022graph, batzner20223}, brain connectivity analysis in neuroscience~\cite{bessadok2022graph,kan2022fbnetgen}, particle physics~\cite{shlomi2020graph}, cybersecurity~\cite{bilot2023graph}). This is possible due to the inherent capability of GNNs in exploring relationships within the structure of the datasets themselves. However, real-world graph datasets are massive in scale. For example, in 2011, the Facebook friendship graph consisted of 721 million users and 69 billion friendship links~\cite{backstrom2012four}. In the fourth quarter of 2024, Meta reported 3.35 billion Daily Active People (DAP) across all of its platforms~\cite{meta2025q4}, showing tremendous growth (e.g., Trillion edges~\cite{ching2015one}) in links across these entities.

Training GNNs in these large graphs encounters several problems (e.g., energy efficiency, scalability, communication overhead, and workload imbalance). The full batch training of the graph is limited by memory, necessitating mini-batch training with neighbor sampling. While this can reduce memory and computational overhead, it can also lead to new problems like increased communication cost. Cai et al.~\cite{cai2021dgcl} show that the communication overhead in distributed GNN training can take from 50\% to 90\% of the training time, which is mainly due to feature communication during the aggregation phase~\cite{shao2024distributed}.

Several mini-batch sampling strategies have been proposed in the literature to help mitigate these problems. For example, GraphSAGE~\cite{hamilton2017inductive} introduces node-wise sampling to limit the number of neighbors aggregated per node. FastGCN~\cite{chen2018fastgcn} and LADIES~\cite{zou2019layer} further refine this idea by employing layer-wise sampling and importance sampling techniques, respectively, to reduce computation while maintaining convergence properties. However, at the beginning of each iteration, the models use these strategies to create a computation graph, which requires features of the nodes. Some of the features are local due to the hosting of the partition of the dataset on that exact machine, while other features remain remote on other machines. When fetching the remote features located in other training machines, the training process often stalls while communicating with the remote machines to extract the features.

Some of the works in the literature specifically address the communication overhead in fetching the remote features during training through modified sampling decisions~\cite{jiang2021communication}. Distributed GNN training frameworks like DistDGL~\cite{zheng2020distdgl} cache one-hop halo (ghost) nodes (with node IDs) to avoid communication overhead. However, fetching a large number of features for multiple GNN layers requires deeper system-level optimization to make the training process more efficient. To address this, the P3 system~\cite{gandhi2021p3} implements pipelining of feature communication with computation to hide latency, while DGCL~\cite{cai2021dgcl} takes a different approach by optimizing data transfer primitives based on workload and network conditions. Building on this foundation, DGCL introduces a communication planner that uses a storage hierarchy to schedule peer-to-peer transfers. These works are complementary as they use sampling decision biasing, scheduling, and dynamic caching to hide or mitigate communication delays. However, there remains room for a strategy that \textit{preemptively} avoids as much communication as possible by using the graph computation block itself. Along with hiding communication through pipelining, it is important to reduce the actual communication to effectively reduce the bottleneck and speed up training. The process must also avoid slowing the convergence of the GNN model.

\begin{figure}[htbp]
    \centering
    \includegraphics[width=0.45\textwidth]{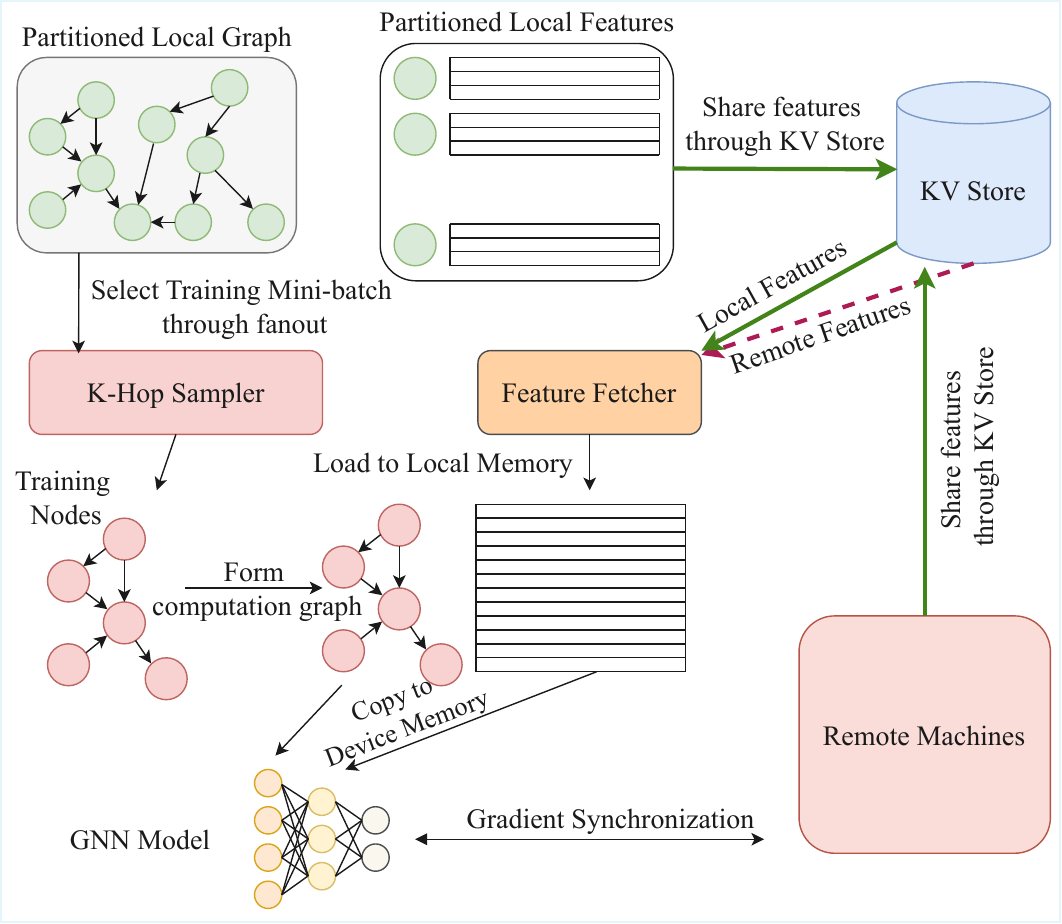} 
    \caption{Feature communication in distributed mini-batch GNN training.}
    \label{fig:comm_overheads}
\end{figure}

Figure \ref{fig:comm_overheads} illustrates the generalized distributed mini-batch GNN training pipeline. The data are partitioned across training machines, and features are shared through a Key-Value (KV) Store abstraction. As the sampled subgraphs form a computation graph, features need to be fetched on the fly to run training inferences, which is the primary bottleneck in training large graph datasets.

To mitigate this bottleneck, we present \name, a novel distributed GNN training framework that aims to minimize communication overhead at its source to improve training throughput and energy efficiency. \name makes the following contributions:
\begin{itemize}
    \item It embeds a fixed-size independent feature cache directly within each worker, avoiding management and network overheads of centralized or fully replicated data stores. The aggregate cache capacity of the system scales horizontally with the number of workers.
    \item It exploits the long-tail distribution of data access in GNNs, where a small subset of high-degree "celebrity" nodes is accessed far more frequently than others. It develops an adaptive, dual-buffer caching policy that prioritizes retaining these frequently accessed node features. This ensures that the most computationally valuable data remains in the local cache across iterations, reducing redundant network traffic for the same features.
    \item It models a highly efficient asynchronous prefetcher that runs concurrently with the training iterations to prepare mini-batches for the next iteration. By maintaining a dynamic queue of requests at the step level, the prefetcher effectively pipelines communication with computation and hides the residual communication latency, reducing the overall training time. 
\end{itemize}

By utilizing these key innovations, \name improves end-to-end training throughput significantly on average over the state-of-the-art models on three benchmark graph datasets by vastly reducing the time spent on fetching the features from remote machines. This is achieved by reducing the actual volume of network communication, along with effective overlapping of I/O operations with computation. We also observe near-linear scalability with increasing the number of computing units in distributed training. Furthermore, \name reduces energy consumption in both CPU and GPU, mainly due to the reduction in training time. 
\section{Background and Related Work}
\subsection{Distributed GNN Training}
A graph can be represented as $G = (V, E)$, where $V = \{v_1, v_2, \ldots, v_n\}$ is the set of nodes and $E \subseteq V \times V$ represents the set of edges. Each node $v_i$ contains a feature vector $x_i \in \mathbb{R}^d$. The complete feature space is denoted by $X \in \mathbb{R}^{ n\times d}$. If the graph is labeled, each node $v_i$ has a corresponding $y_i$ from a label set $Y$.
The computation in a GNN layer can be denoted by:
\begin{equation}
h_v^{(l+1)} = \text{COMB}^{(l)}\left(h_v^{(l)}, 
\text{AGG}^{(l)}\left(\{h_u^{(l)} : u \in \eta(v)\}\right)\right)
\end{equation}

Where the feature vector of node $v$ at layer $l$ is denoted by $h_v^{(l)}$ and the set of its neighbors is $\eta(v)$. The Aggregation function $\text{AGG}^{(l)}$ gathers information from neighboring nodes. Then the combination function $\text{COMB}^{(l)}$ merges the aggregated features with the features of node $v$. The definition of these functions is arbitrary and dependent on specific GNN architectures (e.g., the weighted sum for aggregation in GCN~\cite{kipf2016semi}, mean/max pooling with concatenation in GraphSAGE~\cite{hamilton2017inductive}).

Full‑batch GNN training~\cite{kipf2016semi} quickly exceeds GPU memory on large graphs, leading to \textit{Mini‑batch Sampling}~\cite{hamilton2017inductive}, which builds a much smaller computation graph for every iteration. In the literature, various mini-batch sampling strategies have been proposed. \textit{Node-wise sampling}, as in GraphSAGE~\cite{hamilton2017inductive}, samples a fixed number of neighbors per node to reduce neighborhood explosion, where for each node \( v \) at layer \( l \), a subset \( \widetilde{\eta}(v) = \text{SAMPLE}(\eta(v), k) \) is selected. While it is efficient, it can introduce variance. VR-GCN~\cite{chen2017stochastic} proposes historical activations as control variates. \textit{Layer-wise sampling}, such as FastGCN~\cite{chen2018fastgcn}, samples nodes independently at each layer via importance sampling, while LADIES~\cite{zou2019layer} improves it by enforcing inter-layer connectivity to ensure meaningful contributions from the sampled nodes. In contrast, \textit{subgraph sampling} strategies like ClusterGCN~\cite{chiang2019cluster} and GraphSAINT~\cite{zeng2019graphsaint} form mini-batches by extracting entire induced subgraphs. 

 In distributed GNN training, the graph dataset is partitioned across multiple machines, with each machine storing a subset of nodes and their features. During training, mini-batches often require multi-hop neighbor features, many of which reside on remote partitions. These features are retrieved via Remote Procedure Calls (RPCs), often synchronously, which introduces a significant communication bottleneck leading to stalling of training and GPU under-utilization. Empirical findings show that up to 80\% of training time may be spent on communication and serialization~\cite{gandhi2021p3}.

\subsection{Related Works}
 Some sampling strategies aim to reduce communication overhead by limiting the number of remote nodes sampled through locality-aware sampling. Jiang et al.~\cite {jiang2021communication} skews the neighbor sampling to prioritize local nodes over remote nodes with careful adjustment and ensures that it does not affect convergence much. However, it still has an impact on overall accuracy, and the sampling probabilities are fixed, so it may not adapt well to various configurations. DGS~\cite{wan2022dgs} follows a similar method, but uses an explanation graph to guide the sampling. However, it requires the construction and maintenance of a separate computation graph online that adds overhead and is subject to the performance of the explanation module. 

The primary strategy used in these methods to limit communication is partitioning the data using partitioning algorithms like METIS\cite{karypis1998fast} to minimize edge cuts (used in DistDGL\cite{zheng2020distdgl}) to reduce the dependency on remote partitions. However, limiting the communication between partitions through a partitioning algorithm is an NP-hard problem \cite{bazgan2025dense}. Quantization and compression of feature tensors are used to reduce communication overhead in some works. Sylvie \cite{zhang2023boosting} uses one-bit quantization for gradient and features, AdaQP~\cite{wan2023adaptive} stochastically quantizes features, embeddings, and gradients into low-precision integers, and SC-GNN~\cite{wang2024sc} uses an explanation graph to prioritize semantically important features. These methods usually have an accuracy trade-off and are subject to rigorous experimental validation. For system-level optimization of communication overhead, the P3~\cite{gandhi2021p3} system introduces a pipelining system to hide the communication in the computation background. P3 improves the utilization of resources but does not reduce the total data transferred over the network. Dorylus~\cite{thorpe2021dorylus} is another strategy that offloads GNN training to the CPU and uses asynchronous process management for concurrent executions of the training steps. While using serverless computing for GNN training is innovative, it does not address redundant data transfer over the network.

In CNN training, Clairvoyant \cite{dryden2021clairvoyant} uses a pre-computed fixed linear access sequence for prefetching. This is based on the core assumption that the mini-batch accesses are predictable and sequential with uniform distribution of accesses. This holds for CNN workloads that operate on regular grid-structured datasets. In contrast, GNN training workloads have topology-driven, non-sequential data access where the data access is determined by graph structure and connectivity, resulting in highly skewed, long-tail access distributions. Furthermore, Clairvoyant requires careful coordination between participants through an initial all-gather operation, which is fundamentally non-scalable in GNN training due to the explosion of edge connections in GNN datasets. RapidGNN uses a decentralized cache with constant memory overhead and an adaptive cache update policy to manage the long-tail distribution of node feature requests efficiently. It also ensures scalability by demonstrating increased throughput and speedup as the number of compute machines grows without increasing CPU memory and with a modest, controlled increase in GPU memory due to the device-resident cache.

\subsection{Baseline GraphSAGE Model}
Distributed GNN training frameworks like Deep Graph Library (DGL)\cite{wang2019deep} usually fetch the features needed for an iteration of training by dispatching on-the-fly fetch requests for features of each node, which can result in frequent and redundant RPC calls that can dominate training time \cite{gandhi2021p3}. For our implementation, we use distributed GraphSAGE from DGL to learn a large graph by partitioning it over multiple machines and then using mini-batch training to update the model parameters. The graph is divided into \(G_{i}\) partitions using Random Partition method\cite{zheng2020distdgl} or METIS\cite{karypis1998fast}.  Each partition is assigned to a training machine and is used by that machine as its local graph partition for running the training process of the GNN and updating the model parameters. The number of training workers and partitions should be the same. After partitioning, the partitioned dataset is referenced to the training workers so that each can load their assigned partition. The training device can be either a CPU or a GPU. DistGraph provides an abstraction of the graph partitioning so that the local processes can access the whole graph structure when needed using the neighborhood sampler. Mainly, it is used to fetch the neighborhood information of the seed nodes to build the computation blocks. On the other hand, the KV Store stores the features of the local nodes and provides a backend mechanism through which the training process can fetch the features from the remote partition during the feature aggregation phase of the training. This fetching of the remote nodes' features during training is one of the primary bottlenecks in communication efficiency, as the size of the features is quite large and can stack up due to frequent and redundant access requests. 

The Reddit dataset provided by DGL comprises $232{,}965$ nodes, each represented by a $602$-dimensional feature vector \cite{dgl_reddit_dataset}. In our profiling run with two machines, we found approximately $15{,}000$ nodes to be on the remote partition per batch operation. To estimate the network footprint of the node feature tensor:

\begin{itemize}
  \item \textbf{Node feature tensor size:}\\
  $232{,}965 \times 602 \times 4\,\mathrm{B}
   = 560{,}979{,}720\,\mathrm{B}
   \approx \mathbf{534.99\,\mathrm{MiB}}$.
  
  \item \textbf{Per-batch transfer} (batch size $=1{,}000$, 2-partition setup,\\
  $\approx 15{,}000$ remote nodes/batch):\\
  $15{,}000 \times 602 \times 4\,\mathrm{B}
   = 36{,}120{,}000\,\mathrm{B}
   \approx \mathbf{34.45\,\mathrm{MiB}}$.
  
  \item \textbf{Batches per epoch:}\\
  $\lceil 153{,}431 / 1{,}000\rceil = 154$.
  
  \item \textbf{Total data per epoch:}\\
  $154 \times 34.45\,\mathrm{MiB}
   = 5{,}304.8\,\mathrm{MiB}
   \approx \mathbf{5.18\,\mathrm{GiB}}$.
\end{itemize}

The volume of data transfer can increase substantially when more machines are involved, the dataset is larger, and the batch size increases over a large number of epochs. This highlights the communication overhead in distributed GNN training, where feature data loading can become a significant bottleneck.

Unlike the methods that have been discussed above, which react to communication overhead by partitioning, limiting remote node numbers, and quantization/compression at the cost of accuracy and computation overhead, our novel approach proactively reduces communication volume and redundant data fetching operations by using precomputed feature access patterns. By fixing the seeds, we gain a \textit{priori} knowledge of which remote node features will be needed, when they will be needed, and how often they will be needed to design the caching of the most used remote nodes in bulk operations and reuse them. The prefetching mechanism then pipelines the feeding of the features to the training process for upcoming batches. This transforms the system from being a reactive, on-demand process to a coordinated pipeline, yielding a reduction in the number of RPC calls over the network, substantial speedup in training, with minimal changes to the GNN architectures.

\section{Proposed Work}
\label{sec:proposed}
RapidGNN addresses the communication bottleneck of distributed GNN training during the feature aggregation phase. This, in effect, helps improve the overall training throughput and energy efficiency of the training run as well. With the distributed sampler fixed with a deterministic seed, the sequence of batches and their input nodes is known ahead of time using efficient precomputation by streaming the precomputed data to and from the local storage (SSD). \name\ uses this data and implements a scheduled data path that enumerates the per-epoch batches $\{B_{e}\}_{e=1}^{\epsilon}$ and the union of the fan-out nodes. It then performs a one-shot, vectorized build of a bounded steady cache $C_s$ containing the $n_{\mathrm{hot}}$ most frequently used \emph{remote} nodes and maintains a rolling prefetch queue of size $Q$ that hosts the required features for the upcoming training batches. At runtime, the system is cache-first for the Prefetcher and Prefetcher-First for the training loop. RPC calls to fetch remote node features are made only when there are residual misses by the Prefetcher from the cache and by the Trainer from the Prefetcher. In contrast to reactive caching and pipelining, which use online analysis and are limited by the lookahead window due to simulation of Belady's MIN \cite{zhang2023two} to issue fine-grained RPCs on the critical step path, \name\ fixes the access pattern offline and builds $C_s$ once.  For readability, \(b_i\) denotes the \(i\)-th batch of the current epoch \(e\) (i.e., \(b_i\in B_{e}\)).

\paragraph{Offline enumeration and cache construction.}
Let epoch $e$ comprise $\beta$ batches $B_{e}=\{b_1,\dots,b_\beta\}$, and let $N_i^{e}$ be the input nodes for $b_i$.
We precompute
\begin{equation}
N \;=\; \bigcup_{e=1}^{\epsilon} \; \bigcup_{i=1}^{\beta} \; N_i^{e}, 
\qquad 
N_{\mathrm{remote}} \;=\; N \setminus N_{\mathrm{local}},
\end{equation}
where $N_{\mathrm{local}}$ denotes nodes whose features reside on this worker. We
rank nodes in $N_{\mathrm{remote}}$ by access frequency over $\{B_{e}\}$, and select
\[
N_{\mathrm{cache}} \;=\; \{\, v \in N_{\mathrm{remote}} \mid \textit{freq}(v) \text{ ranks top-}n_{\mathrm{hot}} \,\}.
\]
A single vectorized RPC (\textit{VectorPull}) materializes features for $N_{\mathrm{cache}}$ into the steady device memory cache $C_s$.

\paragraph{Rolling prefetch and execution.}
During epoch $e$, a background Prefetcher hosts features for the next $Q$ batches $\{b_{i+1},\dots,b_{i+Q}\}\subseteq B_{e}$ into a bounded queue. This is managed by dispatching fetch requests asynchronously for the upcoming batch streamed from storage. In parallel (to the prefetching and training), a secondary cache $C_{\mathrm{sec}}$ (Buffer 1) is built for $B_{e+1}$ and swapped into $C_s$ (Buffer 0) at the epoch boundary. For batch $b_i$, the runtime serves $N_i^{e}$ from $C_s$ and the staged queue of the Prefetcher. If a miss set $M_i^{e}\subseteq N_i^{e}$ remains, a \textit{SyncPull} fetches those features from the Distributed KV store into the Prefetcher. If a complete batch is not found in the Prefetcher (due to Prefetcher-Trainer race), the features of that batch are fetched through the default path at the default path fetch time. Then the features are copied to the device, and the usual training step runs.

\paragraph{Invariants and bounds.}
Under a seed-controlled sampler and a static graph, the per-step communication in an epoch equals the miss set by the Prefetcher: the RPC count for $b_i$ is $|M_i^{e}|$. Per-worker device memory is bounded by
\[
\mathrm{Mem}_{\mathrm{device}} \;\le\; 2\,n_{\mathrm{hot}}\cdot d \;+\; Q\,m_{\max}\,d ,
\]
\noindent
where $d$ is feature dimensionality and $m_{\max}=\max_{e,i}|N_i^{e}|$. This is because, along with the Prefetcher, the cache is implemented as a double buffer. The schedule and cache contents are deterministic given the seed, therefore, enabling easier profiling and reproducibility of the experiments.

\begin{algorithm}[!htb]
\caption{\name{}: Deterministic schedule with steady cache and rolling prefetch}
\label{alg:rapidgnn}
\begin{algorithmic}[1]
\Statex \textbf{Input:} Graph $G$; fan-out $F$; epochs $\epsilon$; cache size $n_{\mathrm{hot}}$; prefetch window $Q$
\Statex \textbf{Output:} Parameters $\theta$; per-epoch time $\{t_{e}\}$; per-epoch RPCs $\{rpc_{e}\}$

\State Precompute $\{B_{e}\}_{e=1}^{\epsilon}$ with fan-out $F$
\State $N \gets \bigcup_{e,i} N_i^{e}$;\quad $N_{\mathrm{remote}} \gets N \setminus N_{\mathrm{local}}$
\State $N_{\mathrm{cache}} \gets \textit{TopHot}(N_{\mathrm{remote}},\, n_{\mathrm{hot}},\, \textit{freq})$
\State $C_s \gets \textit{VectorPull}(N_{\mathrm{cache}})$

\For{$e = 1$ \textbf{to} $\epsilon$} \label{ln:epoch-start}
    \State $rpc_{e} \gets 0$;\quad $t_{\text{start}} \gets \textit{Clock}()$
    \If{$e < \epsilon$}
        \State \textit{Parallel:} build $C_{\mathrm{sec}}$ from $B_{e+1}$ \label{ln:sec-build}
    \EndIf
    \State \textit{Parallel:} prefetch features for next $Q$ batches
    \For{each $b_i \in B_{e}$}
        \State hits $\gets \textit{GetFeatureFromCache}(N_i^{e})$
        \If{\textit{miss} on $M_i^{e} = N_i^{e} \setminus$ hits}
            \State \textit{SyncPull}$(M_i^{e})$;\quad $rpc_{e} \gets rpc_{e} + |M_i^{e}|$
        \EndIf
        \State $\theta \gets \textit{Train}(\theta,\, b_i)$
    \EndFor
    \If{$C_{\mathrm{sec}}$ ready} \label{ln:swap}
        \State $C_s \gets $ $C_{\mathrm{sec}}$
    \EndIf
    \State $t_{e} \gets \textit{Clock}() - $ $t_{\text{start}}$
\EndFor
\end{algorithmic}
\end{algorithm}
Algorithm \ref{alg:rapidgnn} presents \name's three-stage schedule: deterministic sampling, hot-set caching, and asynchronous prefetching in a unified workflow. 

\begin{figure*}[!ht]
  \centering
  \includegraphics[width=\textwidth]{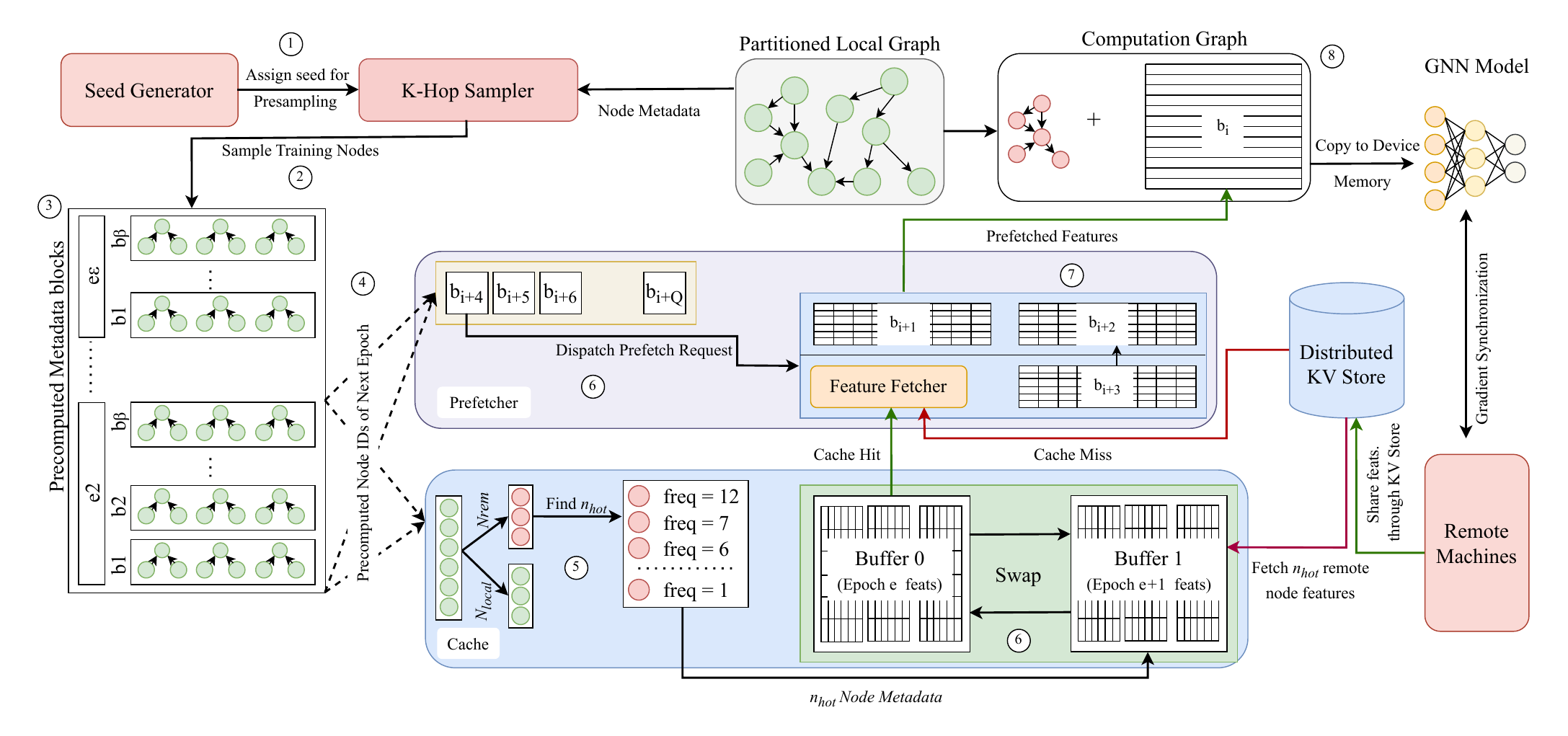}
  \caption{\name\ overview. A seed fixes a K-hop sampler (fan-out $F$) that precomputes, per epoch $e$, the batch set $B_{e}$ and the node IDs they reference. Remote nodes are ranked by frequency; the top-$n_{\mathrm{hot}}$ form $N_{\mathrm{cache}}$ for the steady cache $C_s$ (Buffer 0). A secondary cache $C_{\mathrm{sec}}$ (Buffer 1) is prepared for epoch $e{+}1$ and swapped at the boundary. A \textbf{Prefetcher} stages the next $Q$ batches $\{b_{i+1},\dots,b_{i+Q}\}$, the \textbf{Feature Fetcher} serves hits from cache and issues \textit{SyncPull} on miss sets via the Distributed KV store, and the computation graph for $b_i$ proceeds on the device.}
  \label{fig:rapidgnn}
\end{figure*}

\paragraph{Scalability}
RapidGNN scales with the increase in the number of machines because of bounded per-worker memory, constant per-worker communication, and decentralized coordination. With a growing number of $P$ workers, each worker's device memory is constrained by $\text{Mem}_{\text{device}}$. This bound is strictly dependent on the size of the cache and prefetch queue $(n_{\text{hot}}, Q)$.  Given that an effective partition algorithm maintains a remote node fraction of $c$, each worker's remote feature requests remain proportional to $c \cdot |\text{batch}|$. The cache hit rate $h$ reduces this to $(1-h) \cdot c \cdot |\text{batch}|$ per step. Since both $c$ and $h$ are properties of the graph structure and access patterns rather than $P$, per-worker communication overhead stays bounded as the system scales. One of the significant architectural overheads from Clairvoyant Prefetching \cite{dryden2021clairvoyant} is the initial all-gather operation, which is absent here due to our decentralized precomputation of the complete sampling run. The streaming of the extensively sampled data into the storage system enables running precomputation without increasing CPU memory usage, enabling us to perform this operation on massive graph datasets like OGBN-Papers100M.

\paragraph{Seeding and reproducibility.}
In the precomputation stage, we assign each worker $w$ a PRNG seed via
\[
s_{e,i}^{(w)} = H(s_0, w, e, i),
\]
Where $H$ is a cryptographic hash function and $s_0$ is a global base seed.

\begin{proposition}\label{prop:seeded-sgd}
Let the batches $B_e$ be drawn by a uniform neighbor sampler on $G$ with fan-out $F$ using seeds $s_{e, i}^{(w)}$ above. For any $b_i \in B_e$:
\begin{enumerate}
  \item[(a)] The marginal law of $b_i$ matches that of an online uniform draw with fan-out $F$.
  \item[(b)] For distinct tuples $(e,i,w) \neq (e',i',w')$, the resulting batches are independent.
  \item[(c)] The gradient estimator $g(\theta;b_i) = \nabla_{\theta}\frac{1}{|b_i|}\sum_{v\in b_i}\mathcal{L}_v(\theta)$ is unbiased with positive variance.
\end{enumerate}
\end{proposition}

\begin{proof}
(a) The hash function $H$ maps distinct inputs to statistically independent uniform outputs. Since neighbor sampling uses this PRNG output to make uniform random choices, the resulting batch distribution is identical to online uniform sampling.

(b) Distinct tuples $(e,i,w) \neq (e',i',w')$ produce independent hash values $s_{e,i}^{(w)}$ and $s_{e',i'}^{(w')}$, which seed non-overlapping PRNG streams. Independence of the random bits implies independence of the sampling processes and resulting batches.

(c) By part (a), each $b_i$ has the same distribution as a uniform random sample. Therefore $\mathbb{E}[g(\theta;b_i)] = \nabla_{\theta}\mathcal{L}(\theta)$ by linearity of expectation and the law of total expectation. Positive variance follows from randomness in batch composition under non-degenerate loss functions. This is validated in \ref{subsec:convergence}, where Figure \ref{fig:accuracy_comparison} shows that deterministic sampling achieves identical convergence behavior to baseline methods across multiple datasets.
\end{proof}
\section{Methodology}

We implement \name\ on top of the DGL PyTorch implementation that follows the scheduled data path shown in Fig.~\ref{fig:rapidgnn}. We use epochs $e\!\in\!\{1,\dots,\epsilon\}$, batches $b_i\!\in\!B_{e}$, steady cache $C_s$ (Buffer~0), secondary cache $C_{\mathrm{sec}}$ (Buffer~1), prefetch window $Q$, and hot-set size $n_{\mathrm{hot}}$. We discuss the components of RapidGNN in detail by describing the given illustration.

\paragraph{(1) Seed generator.}
A global base seed $s_0$ is mixed with worker id $w$, epoch $e$, and batch index $i$ to derive $s_{e, i}^{(w)} = H(s_0, w, e, i)$, which is then assigned to the K-Hop sampler for the presampling task. This ensures non-overlapping pseudorandom streams across workers, epochs, and batches while preserving reproducibility (Proposition \ref{prop:seeded-sgd}).

\paragraph{(2) K-Hop sampler.}
Using $s_{e, i}^{(w)}$ and fan-out $F$, the K-hop sampler produces the complete set of batches for training run $B_{e}=\{b_1,\ldots,b_\beta\}$ together with input-node id sets $N_i^{e}$ and locality flags with respect to the partitioned local graph (green) vs.\ remote partitions (pink). The output from the sampler contains only metadata of the batches (ids, offsets, locality) and does not contain the features to form the computation block. This task is offloaded to the Prefetcher, which uses this sampled data to perform the task.

\paragraph{(3) Precomputed metadata blocks.}
The generated “metadata block’’ from the presampling task is streamed directly into the SSD to avoid overloading the CPU memory. It is stored as ordered list of batches in $B_{e}$, the arrays for $\{N_i^{e}\}$, and bitmasks that mark $N_{\mathrm{local}}$ vs.\ $N_{\mathrm{remote}}$. These blocks are streamed at runtime to construct a computation block and in the features gathering operation in our caching and prefetching mechanism. 

\paragraph{(4) Prefetcher Scheduling.}
During training runtime, the Prefetcher uses the precomputed blocks to queue prefetch requests for the subsequent upcoming $Q$ batches $\{b_{i+1},\dots,b_{i+Q}\}$. It dispatches fetch requests by coordinating cache hits and misses to construct the feature blocks to be used during training, and is the core infrastructure in pipelining communication with computation.

\paragraph{(5) Cache candidate selection.}
From $N=\bigcup_{i}N_i^{e}$ and the epoch’s locality flag tensors (computed prior to training), we compute $N_{\mathrm{remote}}=N\setminus N_{\mathrm{local}}$, rank by access frequency $\textit{freq}(\cdot)$ over $B_{e}$, and take the top-$n_{\mathrm{hot}}$ as $N_{\mathrm{cache}}=\{v\in N_{\mathrm{remote}}\mid \textit{freq}(v)\ \text{ranks top-}n_{\mathrm{hot}}\}$. This selects nodes empirically shown to be repeatedly accessed across mini-batches (the long-tail pattern typical of graph workloads) as demonstrated in Figure \ref{fig:frequency}.

\paragraph{(6) Double-buffer cache and swap.}
One-shot vectorized RPC (\textit{VectorPull}) is issued to fetch features for $N_{\mathrm{cache}}$ into $C_s$ (Buffer~0) for epoch $e$. In parallel, the secondary cache $C_{\mathrm{sec}}$ (Buffer~1) is built (while the current training runs and the primary cache serves the Prefetcher). An atomic cache swap operation is performed at the epoch boundary to place the secondary cache as the next primary cache $C_{\mathrm{sec}}\!\rightarrow\!C_s$. The prefetch requests for $\{b_{i+1},\dots,b_{i+Q}\}$ then use the swapped cache to obtain the most remote nodes from the cache consequently.

\paragraph{(7) Feature Fetcher and staging.}
For each dispatched prefetch request, the Feature Fetcher handles the fetch requests by looking up remote node features from the $C_s$. The remaining misses $M_i^{e}\subseteq N_i^{e}$ are fetched through a vectorized \textit{SyncPull} to the distributed KV Store (which is still fetched out of the critical path). The features are kept in preallocated device tensors for efficient device transfer.

\paragraph{(8) Computation graph and training step.}
When the Trainer requires $N_i^{e}$ for batch $b_i$, the feature tensors are immediately available due to the staging of the data by the Prefetcher in the earlier step. The features are then transferred to the device, along with gathering subgraph information from the already stored graph in the memory to form the computation block for training. The rest of the operation includes the usual forward and backward pass cycle in the GNN model and gradient updates.

Sampler$\rightarrow$Prefetcher and Prefetcher$\rightarrow$ Trainer links use lock-free multi-producer, multi-consumer (MPMC) rings. Cache writes happen when \textit{VectorPull} fetches the candidate features. During this time, the queue is bounded by size $Q$, which stalls only when the Trainer lags, allowing it to fill up $Q$ batches of features before the Trainer consumes them. It resumes prefetching as soon as the depth of the prefetch window falls below $Q$ from consumption of prefetched features by the Trainer.

Metadata blocks are streamed sequentially from SSD, while features are fetched directly into device memory by the cache and the Prefetcher. The per-worker device memory follows the bound discussed in Section \ref{sec:proposed}:
$\mathrm{Mem}_{\mathrm{device}} \le 2\,n_{\mathrm{hot}}\cdot d + Q\cdot m_{\max}\cdot d$, with $d$ being the feature dimension and $m_{\max}=\max_{e,i}|N_i^{e}|$. 

We also instrument the system to track communication volume, cache performance, and resource usage for evaluation.
\section{Evaluation}
To evaluate the effectiveness of \name, we conduct extensive experiments on three benchmark datasets and compare them against the SOTA models. Our evaluation aims to quantify the improvements in training speed, communication reduction, and energy efficiency. Additionally, we provide validation of our Proposition \ref{prop:seeded-sgd}, showing that the deterministic precomputation does not impact the accuracy of the models.

Our strategy primarily exploits the long-tail distribution of access to remote nodes. We tally the frequency of access to remote nodes' features and cache the most frequently used remote nodes in an epoch according to the frequency distribution. Figure~\ref{fig:frequency} shows the distribution of how often each remote node's feature is fetched during training. 
\begin{figure}[htbp]
    \centering
    \includegraphics[width=0.48\textwidth]{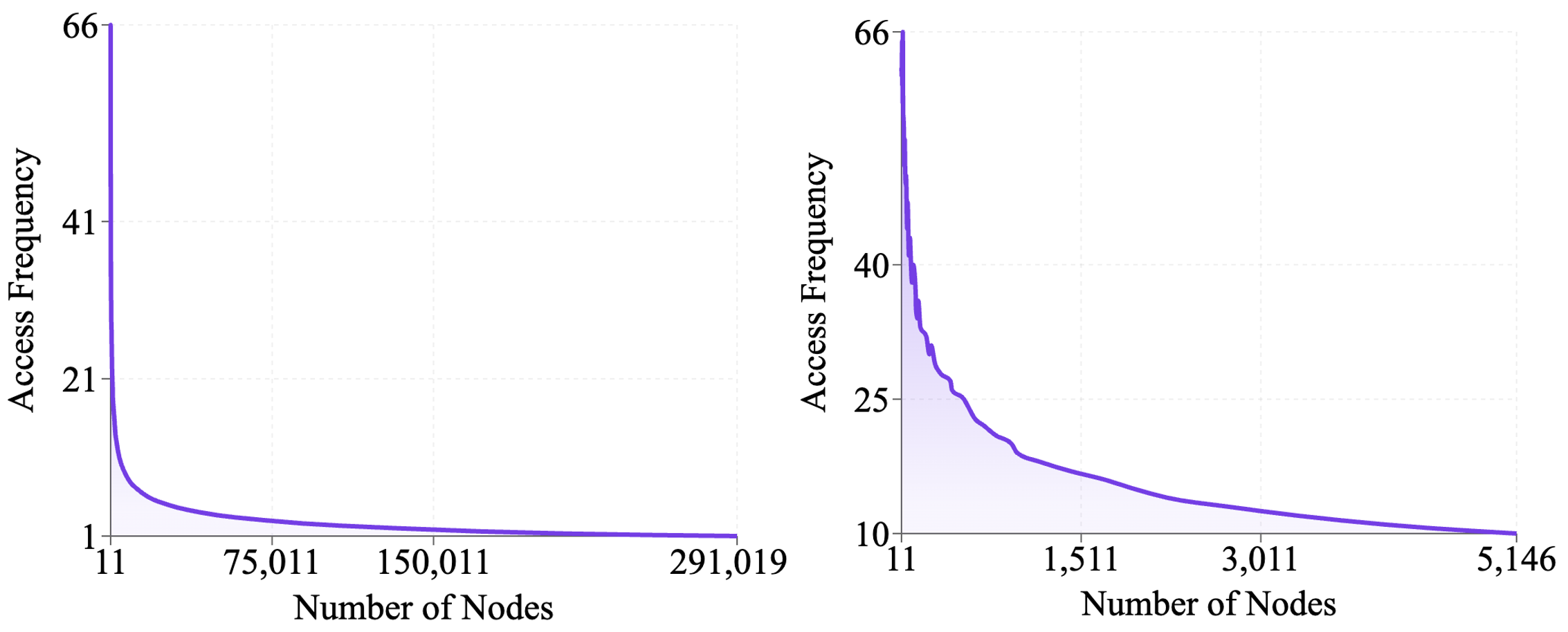}
    \caption{Frequency distribution of remote feature accesses per node. Most nodes are fetched only a handful of times, indicating a long-tail reuse pattern.}
    \label{fig:frequency}
\end{figure}

To demonstrate the frequency distribution, we use the OGBN-Products dataset and sample one of the epochs from our training runs. The analysis reveals a highly skewed power-law distribution in the frequency of node access, with extreme concentration of data at low frequencies. In contrast, the nodes with higher access frequency are relatively lower in count. Nearly half of all nodes (45.3\%) are accessed exactly once, a long tail extending to a maximum frequency of 66, suggesting the presence of highly popular "hub" nodes. This makes the caching decision fairly simple, as instead of real-time graph structure analysis as implemented in other methods, we can tally access frequency in our precomputation phase to guide the caching. The left plot shows the total distribution over the epoch, while the right plot shows the distribution over the top 10\% of the accessed nodes.

\subsection{Experimental Setup}
We perform the experiments on three benchmark graph datasets, which are described in the Table \ref{tab:exp-setup}. The graphs are partitioned with METIS~\cite{karypis1998fast} for RapidGNN, which aims to optimize communication with a balanced edge-cut objective. For baseline, we use a Random partitioner~\cite{hamilton2017inductive} and METIS with DGL for comparison with RapidGNN. These partition schemes allow each machine to host and work with a local graph dataset. We allow one halo hop so that each partition's storage can have the immediate neighbor of its owned node as a ghost node. Each node in these datasets has a high-dimensional feature vector (dense attributes), thus validating the costly feature fetching operation. 

\begin{table}[htbp]
  \centering
    \caption{Experimental Setup: Datasets and Compute Nodes}
  \label{tab:exp-setup}
  \begin{tabular*}{\linewidth}{@{\extracolsep{\fill}}lccc}
    \toprule
    \textbf{Property} & \textbf{Reddit} & \textbf{OGBN-Prod.} & \textbf{OGBN-Papers} \\
    \midrule
    \multicolumn{4}{l}{\textit{Graph Statistics}} \\
    \midrule
    \# Nodes          & 232{,}965     & 2{,}449{,}029   & 111{,}059{,}956 \\
    \# Edges          & 114.8M        & 123.7M          & 1.62B \\
    Feat. Dim.        & 602           & 100             & 128 \\
    \# Classes        & 50            & 47              & 172 \\
    \midrule
    \multicolumn{4}{l}{\textit{Compute Node Specs (per machine, 4 total)}} \\
    \midrule
    Platform          & \multicolumn{3}{c}{Chameleon Cloud} \\
    CPU               & \multicolumn{3}{c}{2× Intel Xeon E5-2670 v3 (24 cores)} \\
    Memory            & \multicolumn{3}{c}{128\,GiB RAM} \\
    GPU               & \multicolumn{3}{c}{2× Tesla P100 (16\,GiB each)} \\
    Storage           & \multicolumn{3}{c}{400–1000\,GB SSD} \\
    Network           & \multicolumn{3}{c}{10\,Gbps Ethernet} \\
    OS                & \multicolumn{3}{c}{Ubuntu 22.04 LTS} \\
    \bottomrule
  \end{tabular*}
\end{table}

We compare our method with three other models - DistDGL GCN~\cite{zheng2020distdgl}, and GraphSAGE~\cite{hamilton2017inductive} model as DGL-Random, and DGL-METIS (using separate partitioning strategy). We use Chameleon Cloud~\cite{keahey2020lessons} GPU nodes to conduct the experiments, which are specified in Table \ref{tab:exp-setup}. We train for 10 epochs in all experiments and report per-epoch and per-step performance metrics. We use the Nvidia NVML~\cite{nvml_lib} and psutil~\cite{psutil_lib} libraries to measure the CPU and GPU metrics during training. 

\subsection{Training Time and Network Efficiency}
\name delivers substantial acceleration across the benchmark datasets and all batch sizes. Table~\ref{tab:speedup} reports the speedup factors relative to GCN, DGL-METIS, and DGL-Random. Averaged over all configurations, \name is \textbf{2.46$\times$}, \textbf{2.26$\times$}, and \textbf{3.00$\times$} faster than DGL-METIS, DGL-Random, and Dist GCN, respectively. The improvement comes from dramatically reducing the waiting time for on-demand feature fetching and using the Prefetcher to feed the features to training, which is still not as high as the network fetch time improvement, as the operations overheads can sometimes lead to Prefetcher and Trainer race. However, we massively reduce the time spent in fetching remote node features over the network (which primarily contributed to the step time reduction) by \textbf{12.70$\times$}, \textbf{9.70$\times$}, and \textbf{15.39$\times$} over DGL-METIS, DGL-Random, and Dist GCN, respectively. This directly supports that caching the hot nodes from long-tail distribution remote nodes reduces the remote fetching, which is highest in the large subgraph construction in Dist GCN.

\begin{table}[htbp]
  \centering
  \setlength{\tabcolsep}{0pt}
  \caption{Speedup of RapidGNN over DGL‑METIS, DGL‑Random, and GCN}
  \label{tab:speedup}
  \begin{tabular*}{\columnwidth}{@{\extracolsep{\fill}} l c ccc ccc}
    \toprule
    \multirow{2}{*}{Dataset} & \multirow{2}{*}{Batch}
      & \multicolumn{3}{c}{Step Speedup} 
      & \multicolumn{3}{c}{Network Speedup} \\
    \cmidrule(lr){3-5}\cmidrule(lr){6-8}
     &  & METIS & Random & GCN 
         & METIS & Random & GCN \\
    \midrule
    OGBN‑Papers   & 1000 & 1.44 & 1.26 & 1.71 & 3.70  & 3.42  & 4.40  \\
                  & 2000 & 1.53 & 1.36 & 1.97 & 4.51  & 3.71  & 5.33  \\
                  & 3000 & 1.65 & 1.48 & 1.90 & 4.26  & 4.10  & 5.25  \\
    OGBN‑Products & 1000 & 1.42 & 1.52 & 1.76 & 5.81  & 6.42  & 6.68  \\
                  & 2000 & 1.32 & 1.34 & 1.57 & 5.19  & 5.36  & 6.08  \\
                  & 3000 & 1.34 & 1.46 & 1.74 & 4.59  & 5.40  & 6.11  \\
    Reddit        & 1000 & 3.97 & 3.41 & 4.89 & 32.57 & 22.25 & 40.51 \\
                  & 2000 & 4.85 & 4.24 & 5.96 & 28.14 & 19.14 & 35.05 \\
                  & 3000 & 4.60 & 4.31 & 5.51 & 25.52 & 17.52 & 29.11 \\
    \midrule
    Average        &  —   & 2.46  & 2.26  & 3.00  & 12.70 & 9.70  & 15.39 \\
    \bottomrule
  \end{tabular*}
\end{table}

We further profiled the actual volume of remote‐feature traffic to validate that our network‐time gains come from reduced data transfer. Figure \ref{fig:avg_rpc} plots the mean data transferred per training step for RapidGNN versus the DGL‑METIS baseline across all three datasets and batch sizes. On OGBN‑Papers, RapidGNN transfers just 1.5MB, 3.1MB, and 4.6MB at batch sizes 1000, 2000, and 3000, roughly 2.8×, 2.7×, and 2.6× less than METIS (4.3MB, 8.3MB, 12.0MB). On the Reddit dataset, we observe even more extensive data savings (due to the highly skewed power-law distribution). RapidGNN moves under 1MB per step (0.3MB, 0.6MB, 0.9MB) versus DGL-METIS's 6.8MB, 10.0MB, and 14.0MB, a 15–23× reduction in network load. On OGBN‑Products, we see a 2.2–2.5× drop (2.0MB/3.8MB/5.4MB vs. 4.8MB/8.8MB/12.1MB).

\begin{figure}[!htbp]
    \centering
    \includegraphics[width=0.48\textwidth]{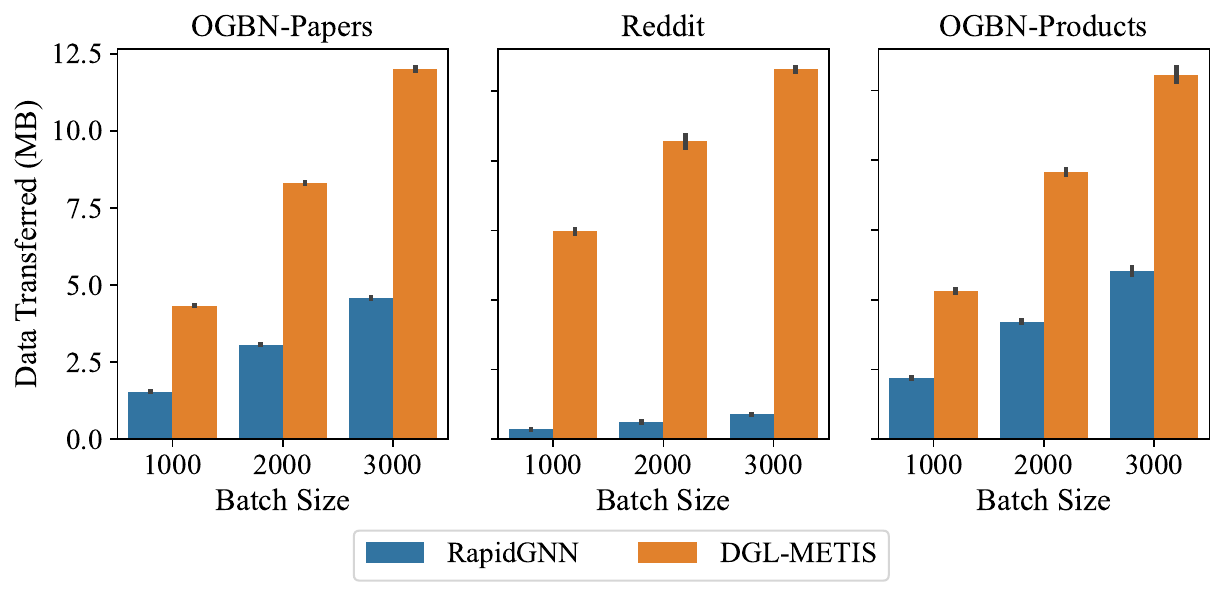}
    \caption{Mean data transfer overhead comparison for graph partitioning method for RapidGNN versus DGL-METIS baseline across three benchmark datasets and varying batch sizes.}
    \label{fig:avg_rpc}
\end{figure}

To validate the effectiveness of cache in reducing the communication volume, we profile a short run of 40 epochs with two machines training the OGBN-Products dataset. We show the findings in Figure \ref{fig:cache_size}, which explains the step-time and epoch-time gains observed in Table \ref{tab:speedup}.

\begin{figure}[!htbp]
    \centering
    \includegraphics[width=0.42\textwidth]{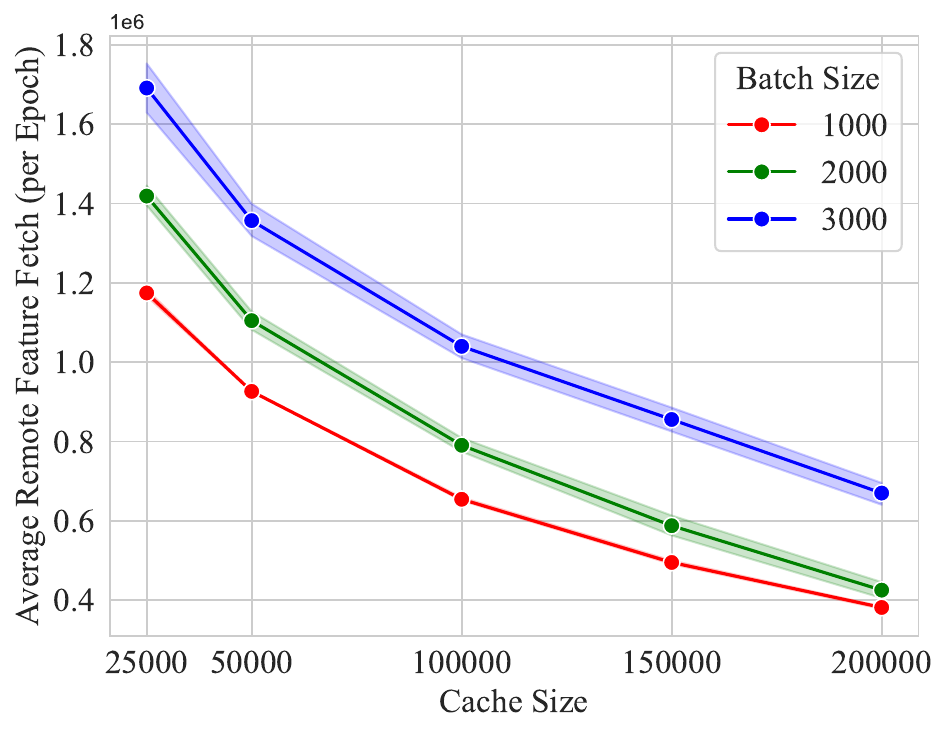}
    \caption{Average number of remote feature fetches per epoch versus cache size}
    \label{fig:cache_size}
    \vspace{-1mm}
\end{figure}
As cache size increases, the average number of remote feature fetches per epoch drops sharply across all batch sizes, demonstrating RapidGNN's ability to capture the high-frequency "hub" nodes in its steady cache. This directly reduces the number of synchronous RPCs issued on the critical path, lowering network latency and minimizing stalls in the training loop. The effect is most pronounced in the low-to-moderate cache range, where small memory allocations yield disproportionately significant reductions in fetch volume, reflecting the long-tail access pattern exploited by RapidGNN. Beyond a certain point, the curve flattens, indicating diminishing returns and enabling practical cache-size selection without excessive memory overhead. By cutting redundant network pulls at the source, RapidGNN sustains high cache-hit mass, improves compute/communication overlap, and converts these micro-level savings into the macro-level throughput improvements reported in the training-time results.

\subsection{Resource Usage and Scalability}
\label{subsec:scale}
 As the number of machines increases, the time required to train an epoch decreases as the workload is distributed across the machines. Therefore, the average epoch time across machines should theoretically reduce as the number of machines increases. In Figure \ref{fig:scalability}, we demonstrate the throughput scaling of RapidGNN with an increasing number of training machines.
\begin{figure}[htbp]
    \centering
    \includegraphics[width=0.42\textwidth]{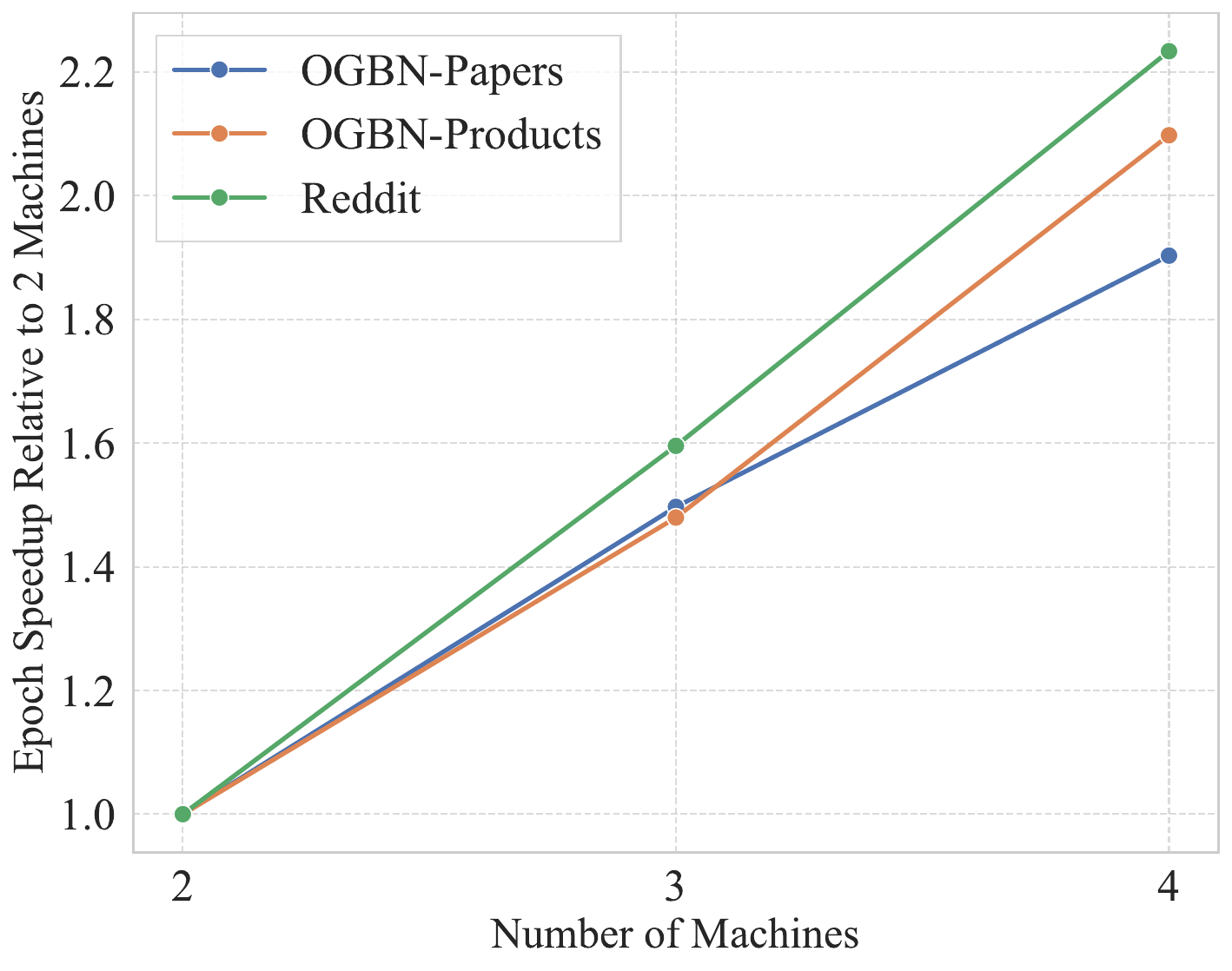}
    \caption{Scalability of RapidGNN across three datasets (OGBN-Papers100M, OGBN-Products, and Reddit)}
    \label{fig:scalability}
\end{figure}
RapidGNN consistently scales linearly as the number of machines increases. For three machines we observe \textbf{1.5} for OGBN-Products to \textbf{1.6} times speedup for the Reddit dataset over a 2-machine setup. For four machines, we observe \textbf{1.7}  to \textbf{2.1} times speedup for the Reddit dataset.

RapidGNN also showed stable resource usage with an increasing number of machines, as shown in Figure \ref{fig:memory_scaling_graph}. GPU memory was consistently higher than the baseline method due to constant cache usage. The CPU memory stayed nearly identical to the baseline method, as our streaming of precomputed data from SSD alleviates any precomputation overhead in CPU memory.

\begin{figure}[!htbp]
    \centering
    \includegraphics[width=0.48\textwidth]{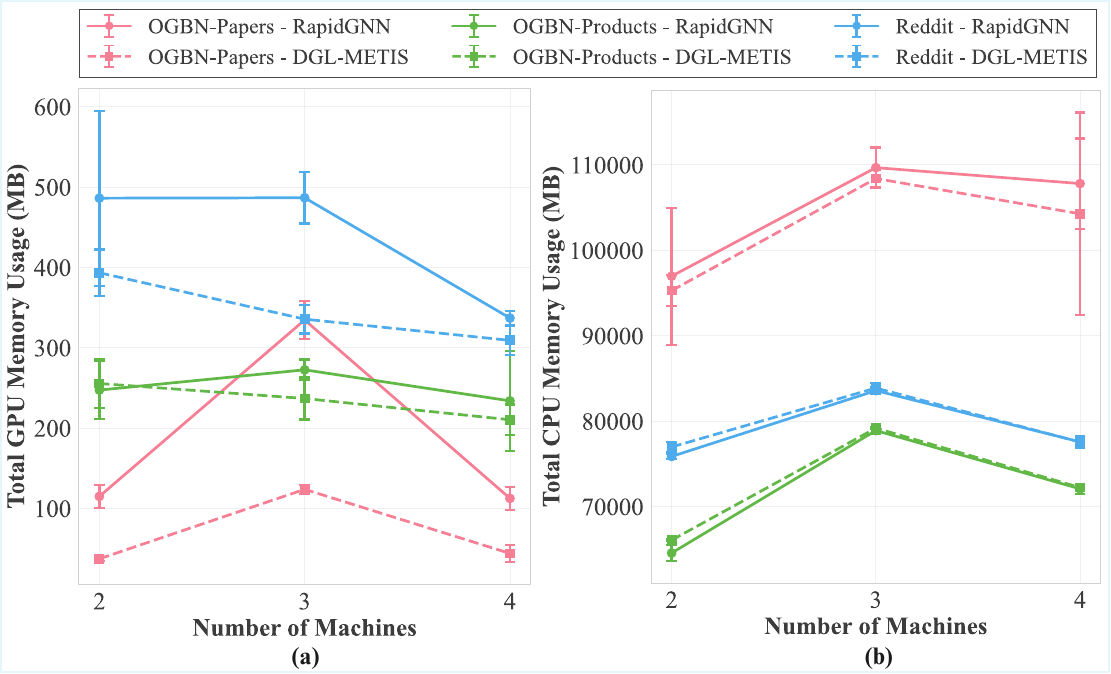}
    \caption{Stable memory scaling of RapidGNN along with DGL-METIS across three datasets (OGBN-Papers100M, OGBN-Products, and Reddit)}
    \label{fig:memory_scaling_graph}
\end{figure}

In Figure \ref{fig:memory_scaling_graph}, we compare the memory usage of \name with DGL-METIS across the benchmark datasets. Figure~\ref{fig:memory_scaling_graph}(a) reports total GPU memory: \name uses more memory because cached and prefetched features are stored on device, but the usage scales smoothly across datasets. Figure~\ref{fig:memory_scaling_graph}(b) reports total CPU memory, which closely tracks the baseline. Overall, \name maintains predictable scaling while trading modest additional GPU memory for reduced communication.

Along with reducing communication overhead and training time, \name is scalable for massive graphs and an increasing number of computing units. In our training run, RapidGNN improved training time with an increasing number of machines and showed consistent improvement in speedup over the best-performing baseline DGL-METIS method, as shown in Figure \ref{fig:scalability}.

\subsection{Energy Efficiency}
As a direct outcome of our highly efficient data transfer scheduling and improved training throughput, \name improves energy efficiency. We measure the energy consumption for batch size 3000 for the OGBN-Products dataset over 10 epochs across three training machines and average it in Table~\ref{tab:energy_metrics_cpu_gpu_first}.

\begin{table}[htbp]
  \centering
  \setlength{\tabcolsep}{0pt}
  \caption{Detailed energy and performance metrics for CPU and GPU components of RapidGNN and DGL-METIS.}
  \label{tab:energy_metrics_cpu_gpu_first}
  \begin{tabular*}{\columnwidth}{@{\extracolsep{\fill}} lcccc}
    \toprule
    \multirow{2}{*}{Metric} 
      & \multicolumn{2}{c}{CPU} 
      & \multicolumn{2}{c}{GPU} \\
    \cmidrule(lr){2-3}\cmidrule(lr){4-5}
      & RapidGNN & DGLM & RapidGNN & DGLM \\
    \midrule
    Total Energy (J)        & 1376.16 & 2464.64 & 2309.52 & 3400.74 \\
    Mean Energy/Epoch (J)   & 137.62  & 246.46  & 230.95  & 340.07  \\
    Min Energy/Epoch (J)    & 117.79  & 234.64  & 219.32  & 332.30  \\
    Max Energy/Epoch (J)    & 164.26  & 258.62  & 260.52  & 361.06  \\
    Mean Power (W)          & 36.73   & 42.70   & 30.84   & 29.45   \\
    Total Duration (s)      & 37.47   & 57.74   & 37.47   & 57.74   \\
    Mean Duration/Epoch (s) & 3.75    & 5.77    & 3.75    & 5.77    \\
    \bottomrule
  \end{tabular*}
\end{table}

Though the implementation of the feature cache in the GPU increases the memory usage, it remains stable throughout the training and even with an increasing number of machines (as validated in \ref{subsec:scale}). Therefore, on the GPU, even though RapidGNN uses slightly more GPU power compared to DGL-METIS (only \textbf{4.7\%} more), it consumes about \textbf{one-third} less total energy (2309 J vs. 3401 J). 

\begin{figure}[!htbp]
    \centering
    \includegraphics[width=0.48\textwidth]{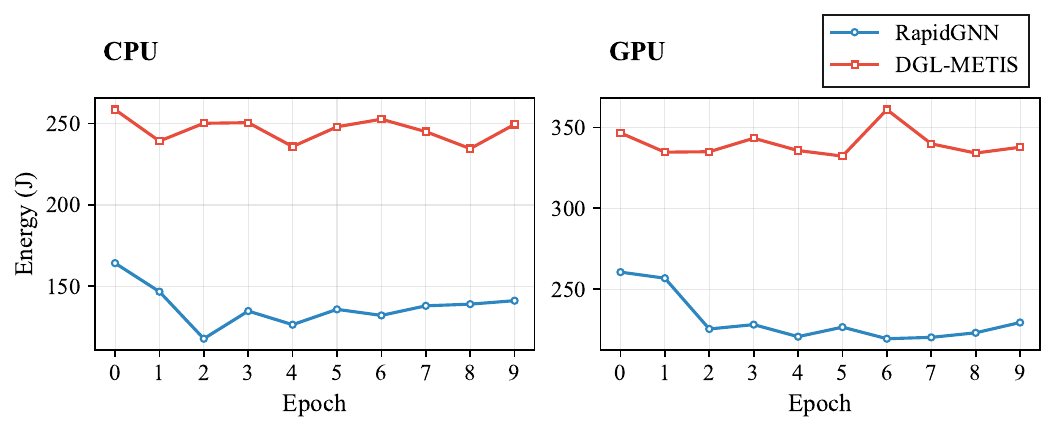}
    \caption{Comparison of CPU and GPU energy consumption in RapidGNN vs DGL-METIS}
    \label{fig:energy}
\end{figure}

On the CPU, RapidGNN reduces total energy by nearly \textbf{half} (1376 J vs. 2465 J) mainly because it completes training ~35\% faster (37.5 s vs. 57.7 s). However, in contrast to GPU, which does the same computation operation on top of hosting features for cache,  its mean CPU power draw is lower (36.73 W vs. 42.70 W, about \textbf{14\%} less), which means it saves energy both by drawing less power and by running for a shorter duration. This leads to proportionally lower per-epoch energy, with the minimum and maximum CPU energy per epoch also substantially lower than DGL-METIS. This is due to the CPU not spending intermittent repetitive work on constructing batches on the fly in between GPU workloads, along with redundant marshalling of data, handling network I/O, and context‐switching between communication and computation as a direct benefit of RapidGNN.

\subsection{Convergence Evidence}
To verify given Proposition \ref{prop:seeded-sgd}, we compare epoch‐wise training accuracy of \name against the baselines in Figure~\ref{fig:accuracy_comparison}.
\label{subsec:convergence}
\begin{figure}[!htbp]
  \centering

  \includegraphics[width=0.45\textwidth]{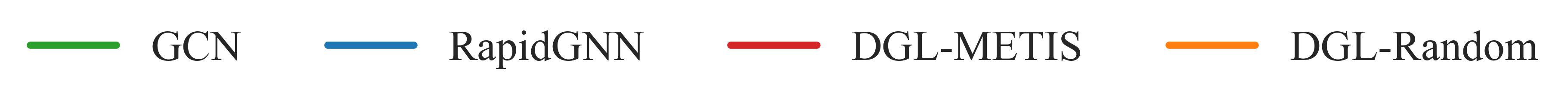}

  \begin{subfigure}[b]{0.48\columnwidth}
    \includegraphics[width=\linewidth]{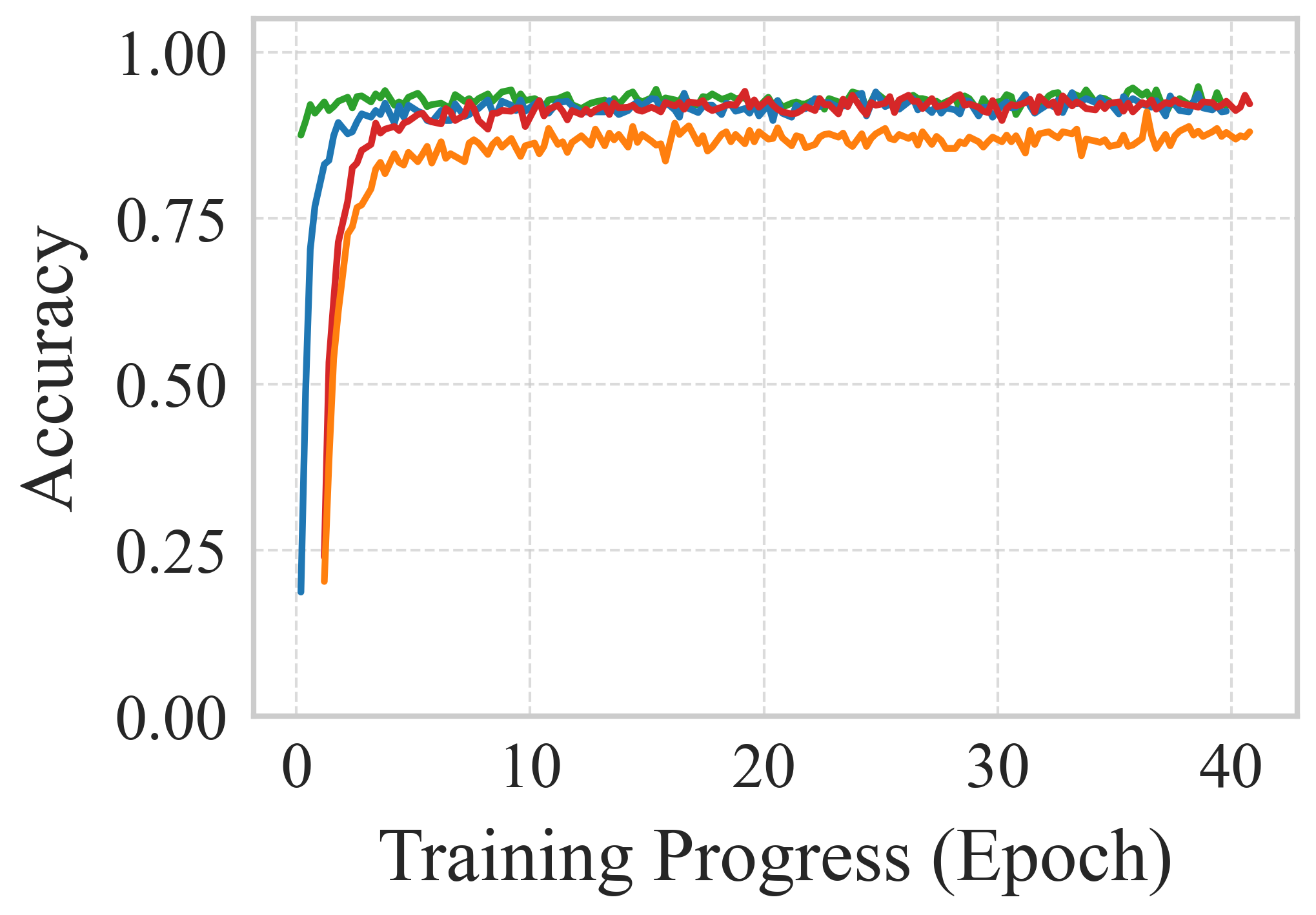}
    \caption{OGBN‑Products, batch size 1000}
    \label{fig:acc_ogbn_1000}
  \end{subfigure}\hfill
  \begin{subfigure}[b]{0.48\columnwidth}
    \includegraphics[width=\linewidth]{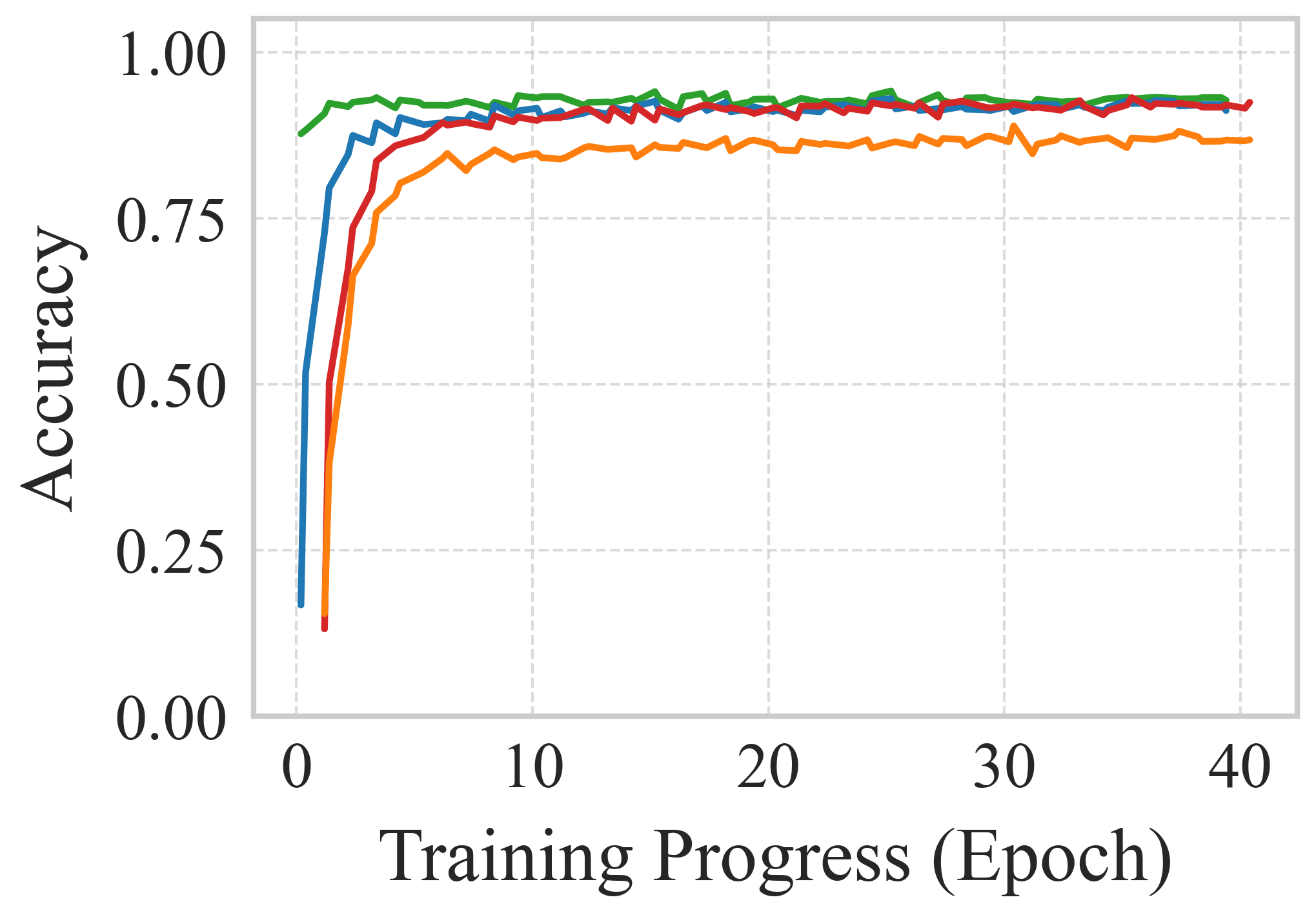}
    \caption{OGBN‑Products, batch size 2000}
    \label{fig:acc_ogbn_2000}
  \end{subfigure}

  \vspace{1mm}

  \begin{subfigure}[b]{0.48\columnwidth}
    \includegraphics[width=\linewidth]{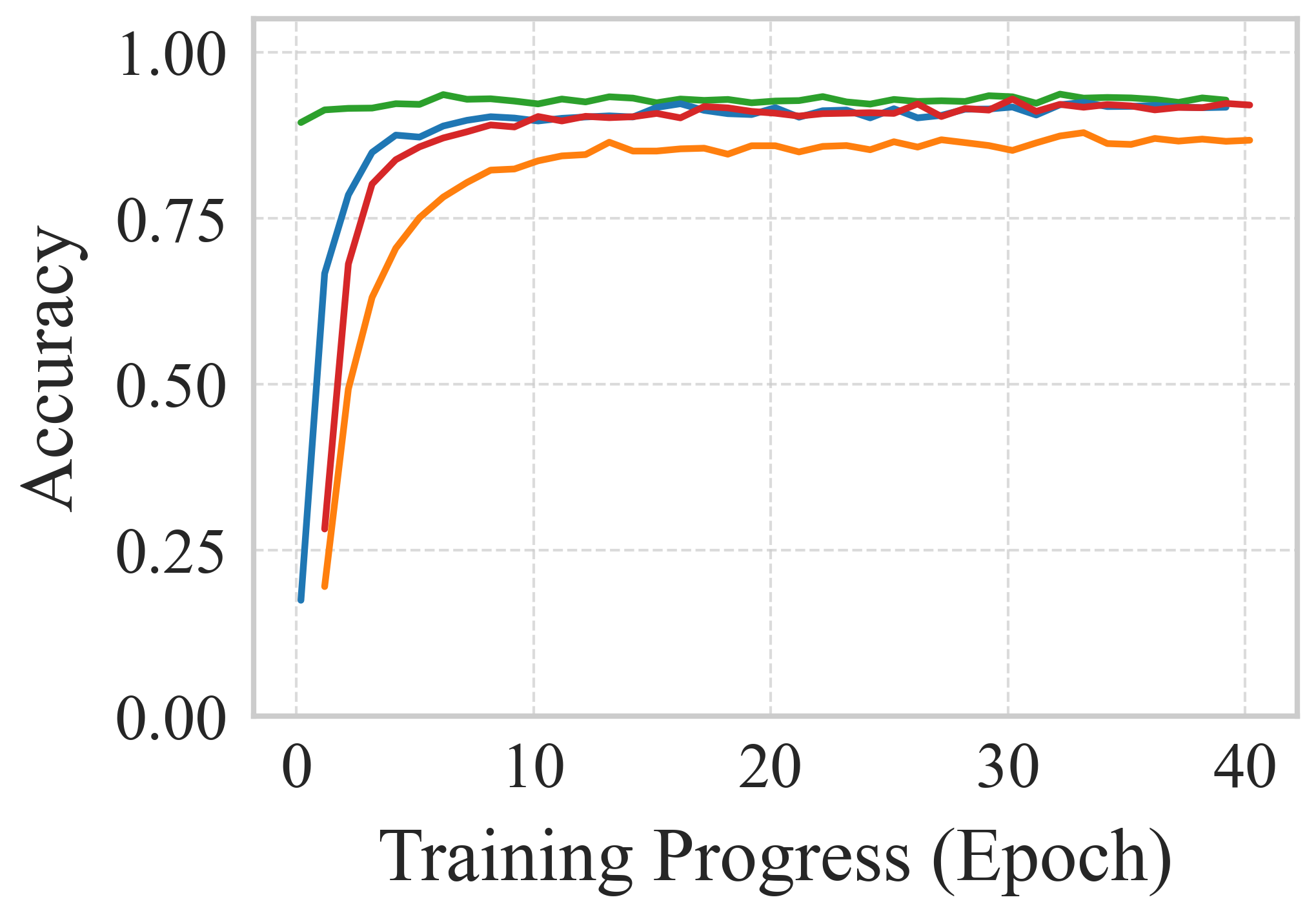}
    \caption{OGBN‑Products, batch size 3000}
    \label{fig:acc_ogbn_3000}
  \end{subfigure}\hfill
  \begin{subfigure}[b]{0.48\columnwidth}
    \includegraphics[width=\linewidth]{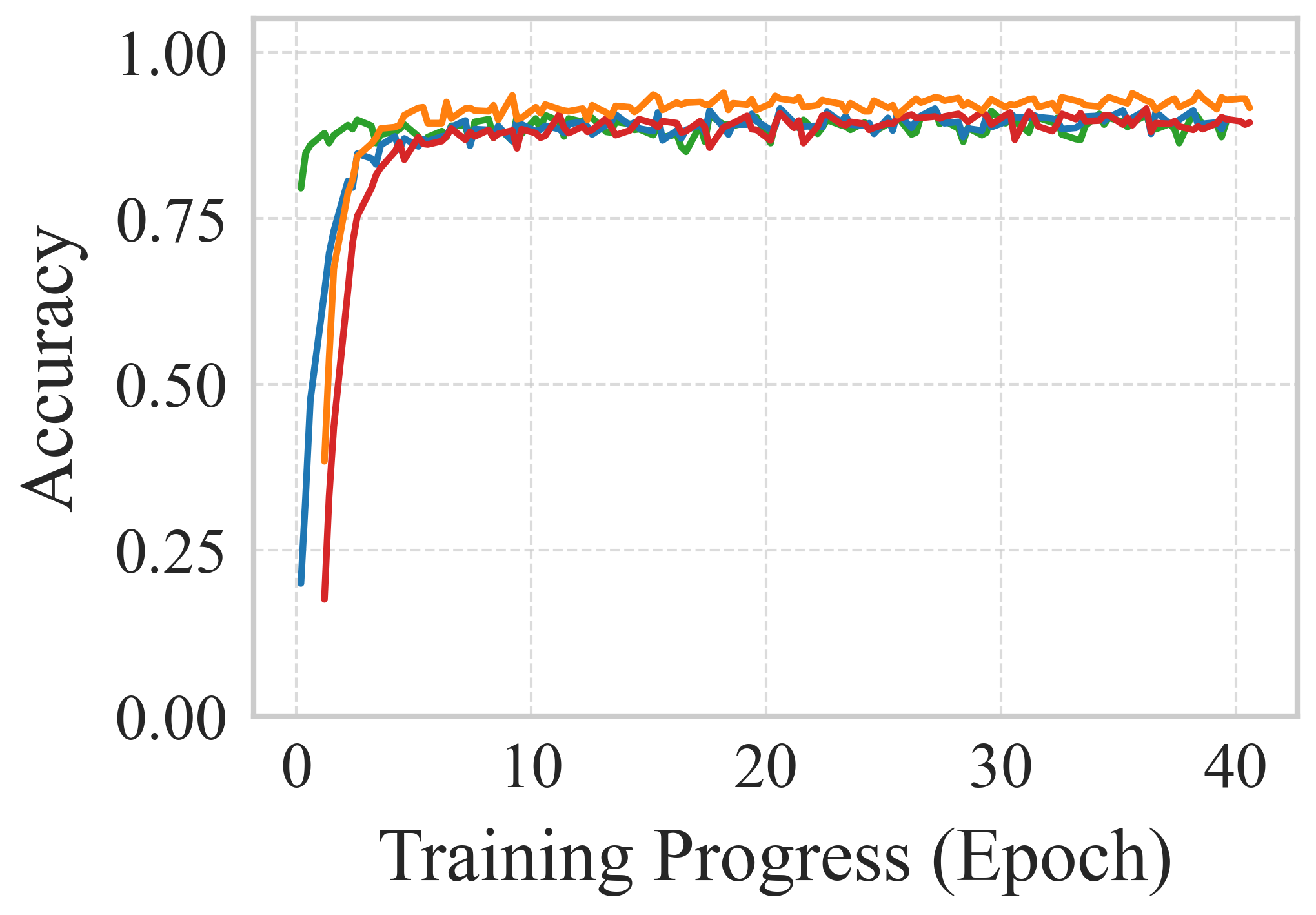}
    \caption{Reddit, batch size 1000}
    \label{fig:acc_reddit_1000}
  \end{subfigure}

  \vspace{1mm}

  \begin{subfigure}[b]{0.48\columnwidth}
    \includegraphics[width=\linewidth]{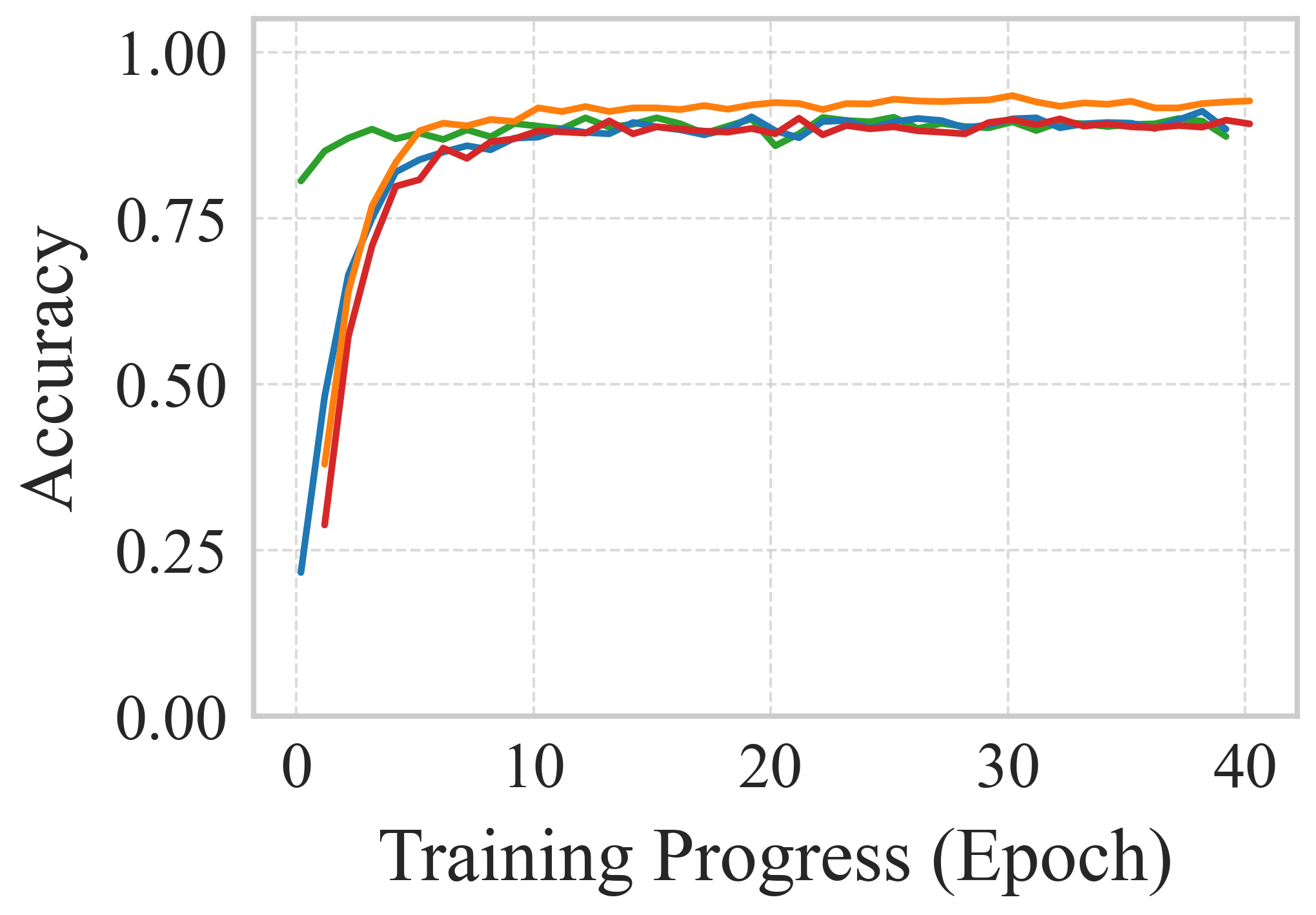}
    \caption{Reddit, batch size 2000}
    \label{fig:acc_reddit_2000}
  \end{subfigure}\hfill
  \begin{subfigure}[b]{0.48\columnwidth}
    \includegraphics[width=\linewidth]{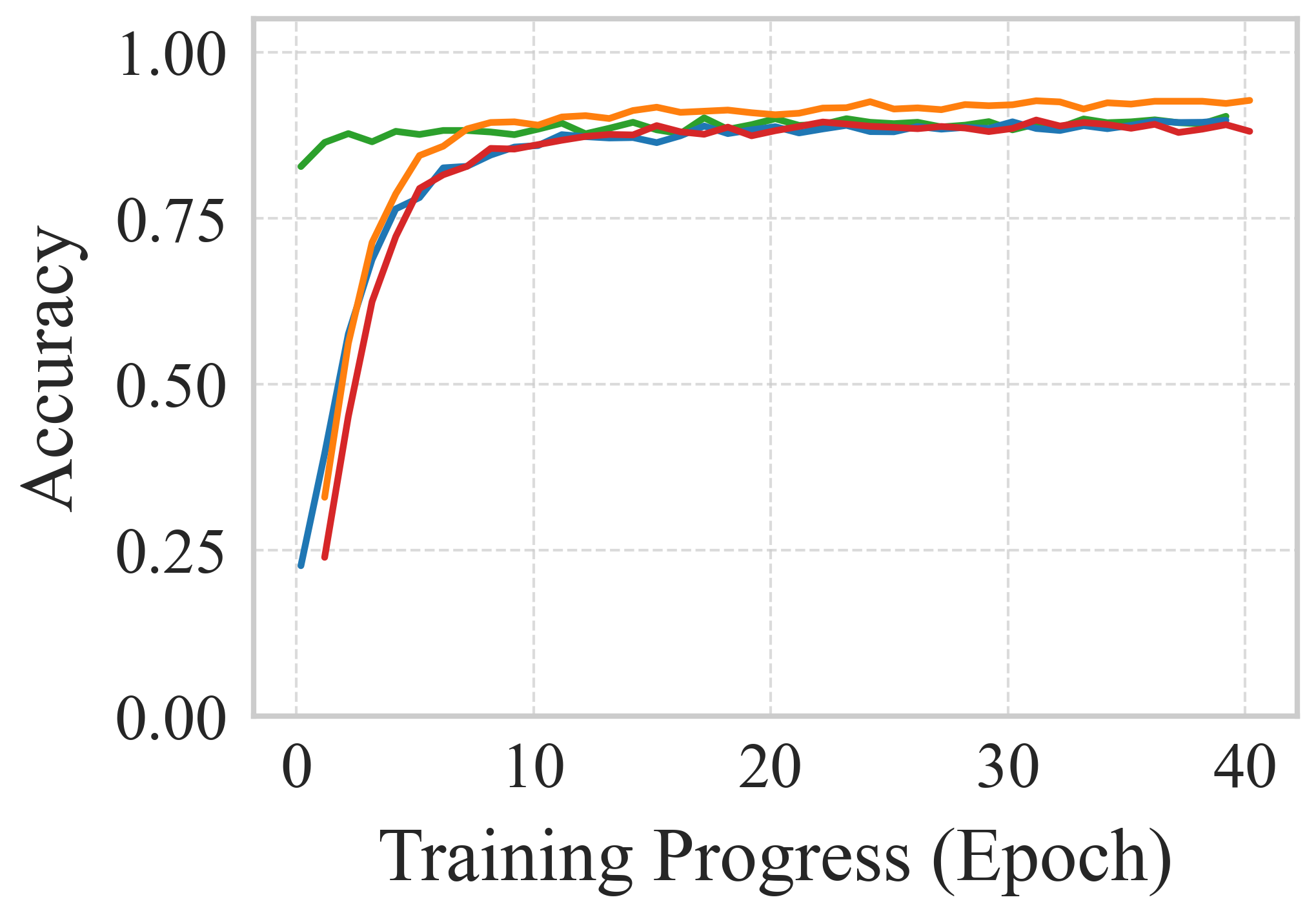}
    \caption{Reddit, batch size 3000}
    \label{fig:acc_reddit_3000}
  \end{subfigure}

  \vspace{-1.5mm} 
  \caption{Training accuracy across batch sizes on OGBN‑Products (top three) and Reddit (bottom three).}
  \label{fig:accuracy_comparison}
  \vspace{-4mm}
\end{figure}

In all six configurations, \name's accuracy curves rapidly rise and plateau at the same level as the baselines. We observe no signs of slowed convergence or increased variance due to deterministic sampling or cache‐guided prefetching. These results empirically confirm Proposition \ref{prop:seeded-sgd}: fixing the PRNG seed and employing a hot‐node cache do not bias or destabilize the stochastic gradient estimates, preserving the convergence guarantees of standard mini-batch SGD.

\section{Conclusion}

We present \name, an access pattern-based cache optimization method and prefetching technique for distributed GNN training. It significantly improves communication overhead, energy efficiency, and training time without compromising the model's convergence rate by actively reducing communication and reusing features. Our implementation requires minimal changes within the distributed DGL framework and uses existing modules to build the \name architecture while achieving significant improvements. On three respective benchmark graphs, we demonstrate reduction in overall training time along with reduction of data transferred over the network. We also demonstrate a significant reduction in total GPU and CPU energy consumption and CPU power usage. In the future, we aim to extend this architecture to other GNN architectures and dynamic graph workloads, as our method does not require any modification to existing architectures. We also plan to analyze the performance and energy consumption trade-offs further and design predictive system-level optimizations to increase communication efficiency with a minimum energy footprint.

\bibliographystyle{ACM-Reference-Format}
\bibliography{references} 


\begin{thebibliography}{41}


\ifx \showCODEN    \undefined \def \showCODEN     #1{\unskip}     \fi
\ifx \showISBNx    \undefined \def \showISBNx     #1{\unskip}     \fi
\ifx \showISBNxiii \undefined \def \showISBNxiii  #1{\unskip}     \fi
\ifx \showISSN     \undefined \def \showISSN      #1{\unskip}     \fi
\ifx \showLCCN     \undefined \def \showLCCN      #1{\unskip}     \fi
\ifx \shownote     \undefined \def \shownote      #1{#1}          \fi
\ifx \showarticletitle \undefined \def \showarticletitle #1{#1}   \fi
\ifx \showURL      \undefined \def \showURL       {\relax}        \fi
\providecommand\bibfield[2]{#2}
\providecommand\bibinfo[2]{#2}
\providecommand\natexlab[1]{#1}
\providecommand\showeprint[2][]{arXiv:#2}

\bibitem[Backstrom et~al\mbox{.}(2012)]%
        {backstrom2012four}
\bibfield{author}{\bibinfo{person}{Lars Backstrom}, \bibinfo{person}{Paolo Boldi}, \bibinfo{person}{Marco Rosa}, \bibinfo{person}{Johan Ugander}, {and} \bibinfo{person}{Sebastiano Vigna}.} \bibinfo{year}{2012}\natexlab{}.
\newblock \showarticletitle{Four degrees of separation}. In \bibinfo{booktitle}{\emph{Proceedings of the 4th annual ACM Web science conference}}. \bibinfo{pages}{33--42}.
\newblock


\bibitem[Batzner et~al\mbox{.}(2022)]%
        {batzner20223}
\bibfield{author}{\bibinfo{person}{Simon Batzner}, \bibinfo{person}{Albert Musaelian}, \bibinfo{person}{Lixin Sun}, \bibinfo{person}{Mario Geiger}, \bibinfo{person}{Jonathan~P Mailoa}, \bibinfo{person}{Mordechai Kornbluth}, \bibinfo{person}{Nicola Molinari}, \bibinfo{person}{Tess~E Smidt}, {and} \bibinfo{person}{Boris Kozinsky}.} \bibinfo{year}{2022}\natexlab{}.
\newblock \showarticletitle{E (3)-equivariant graph neural networks for data-efficient and accurate interatomic potentials}.
\newblock \bibinfo{journal}{\emph{Nature communications}} \bibinfo{volume}{13}, \bibinfo{number}{1} (\bibinfo{year}{2022}), \bibinfo{pages}{2453}.
\newblock


\bibitem[Bazgan et~al\mbox{.}(2025)]%
        {bazgan2025dense}
\bibfield{author}{\bibinfo{person}{Cristina Bazgan}, \bibinfo{person}{Katrin Casel}, {and} \bibinfo{person}{Pierre Cazals}.} \bibinfo{year}{2025}\natexlab{}.
\newblock \showarticletitle{Dense graph partitioning on sparse and dense graphs}.
\newblock \bibinfo{journal}{\emph{J. Comput. System Sci.}} (\bibinfo{year}{2025}), \bibinfo{pages}{103619}.
\newblock


\bibitem[Bessadok et~al\mbox{.}(2022)]%
        {bessadok2022graph}
\bibfield{author}{\bibinfo{person}{Alaa Bessadok}, \bibinfo{person}{Mohamed~Ali Mahjoub}, {and} \bibinfo{person}{Islem Rekik}.} \bibinfo{year}{2022}\natexlab{}.
\newblock \showarticletitle{Graph neural networks in network neuroscience}.
\newblock \bibinfo{journal}{\emph{IEEE Transactions on Pattern Analysis and Machine Intelligence}} \bibinfo{volume}{45}, \bibinfo{number}{5} (\bibinfo{year}{2022}), \bibinfo{pages}{5833--5848}.
\newblock


\bibitem[Bilot et~al\mbox{.}(2023)]%
        {bilot2023graph}
\bibfield{author}{\bibinfo{person}{Tristan Bilot}, \bibinfo{person}{Nour El~Madhoun}, \bibinfo{person}{Khaldoun Al~Agha}, {and} \bibinfo{person}{Anis Zouaoui}.} \bibinfo{year}{2023}\natexlab{}.
\newblock \showarticletitle{Graph neural networks for intrusion detection: A survey}.
\newblock \bibinfo{journal}{\emph{IEEE Access}}  \bibinfo{volume}{11} (\bibinfo{year}{2023}), \bibinfo{pages}{49114--49139}.
\newblock


\bibitem[Cai et~al\mbox{.}(2021)]%
        {cai2021dgcl}
\bibfield{author}{\bibinfo{person}{Zhenkun Cai}, \bibinfo{person}{Xiao Yan}, \bibinfo{person}{Yidi Wu}, \bibinfo{person}{Kaihao Ma}, \bibinfo{person}{James Cheng}, {and} \bibinfo{person}{Fan Yu}.} \bibinfo{year}{2021}\natexlab{}.
\newblock \showarticletitle{DGCL: An efficient communication library for distributed GNN training}. In \bibinfo{booktitle}{\emph{Proceedings of the Sixteenth European Conference on Computer Systems}}. \bibinfo{pages}{130--144}.
\newblock


\bibitem[Chen et~al\mbox{.}(2018)]%
        {chen2018fastgcn}
\bibfield{author}{\bibinfo{person}{Jie Chen}, \bibinfo{person}{Tengfei Ma}, {and} \bibinfo{person}{Cao Xiao}.} \bibinfo{year}{2018}\natexlab{}.
\newblock \showarticletitle{Fastgcn: fast learning with graph convolutional networks via importance sampling}.
\newblock \bibinfo{journal}{\emph{arXiv preprint arXiv:1801.10247}} (\bibinfo{year}{2018}).
\newblock


\bibitem[Chen et~al\mbox{.}(2017)]%
        {chen2017stochastic}
\bibfield{author}{\bibinfo{person}{Jianfei Chen}, \bibinfo{person}{Jun Zhu}, {and} \bibinfo{person}{Le Song}.} \bibinfo{year}{2017}\natexlab{}.
\newblock \showarticletitle{Stochastic training of graph convolutional networks with variance reduction}.
\newblock \bibinfo{journal}{\emph{arXiv preprint arXiv:1710.10568}} (\bibinfo{year}{2017}).
\newblock


\bibitem[Chiang et~al\mbox{.}(2019)]%
        {chiang2019cluster}
\bibfield{author}{\bibinfo{person}{Wei-Lin Chiang}, \bibinfo{person}{Xuanqing Liu}, \bibinfo{person}{Si Si}, \bibinfo{person}{Yang Li}, \bibinfo{person}{Samy Bengio}, {and} \bibinfo{person}{Cho-Jui Hsieh}.} \bibinfo{year}{2019}\natexlab{}.
\newblock \showarticletitle{Cluster-gcn: An efficient algorithm for training deep and large graph convolutional networks}. In \bibinfo{booktitle}{\emph{Proceedings of the 25th ACM SIGKDD international conference on knowledge discovery \& data mining}}. \bibinfo{pages}{257--266}.
\newblock


\bibitem[Ching et~al\mbox{.}(2015)]%
        {ching2015one}
\bibfield{author}{\bibinfo{person}{Avery Ching}, \bibinfo{person}{Sergey Edunov}, \bibinfo{person}{Maja Kabiljo}, \bibinfo{person}{Dionysios Logothetis}, {and} \bibinfo{person}{Sambavi Muthukrishnan}.} \bibinfo{year}{2015}\natexlab{}.
\newblock \showarticletitle{One trillion edges: Graph processing at facebook-scale}.
\newblock \bibinfo{journal}{\emph{Proceedings of the VLDB Endowment}} \bibinfo{volume}{8}, \bibinfo{number}{12} (\bibinfo{year}{2015}), \bibinfo{pages}{1804--1815}.
\newblock


\bibitem[{DGL Team}(2024)]%
        {dgl_reddit_dataset}
\bibfield{author}{\bibinfo{person}{{DGL Team}}.} \bibinfo{year}{2024}\natexlab{}.
\newblock \bibinfo{title}{{RedditDataset — DGL 2.5 documentation}}.
\newblock \bibinfo{howpublished}{\url{https://tinyurl.com/58u8tjsr}}.
\newblock
\newblock
\shownote{Accessed: 2025-04-24}.


\bibitem[Dryden et~al\mbox{.}(2021)]%
        {dryden2021clairvoyant}
\bibfield{author}{\bibinfo{person}{Nikoli Dryden}, \bibinfo{person}{Roman B{\"o}hringer}, \bibinfo{person}{Tal Ben-Nun}, {and} \bibinfo{person}{Torsten Hoefler}.} \bibinfo{year}{2021}\natexlab{}.
\newblock \showarticletitle{Clairvoyant prefetching for distributed machine learning I/O}. In \bibinfo{booktitle}{\emph{Proceedings of the International Conference for High Performance Computing, Networking, Storage and Analysis}}. \bibinfo{pages}{1--15}.
\newblock


\bibitem[Gandhi and Iyer(2021)]%
        {gandhi2021p3}
\bibfield{author}{\bibinfo{person}{Swapnil Gandhi} {and} \bibinfo{person}{Anand~Padmanabha Iyer}.} \bibinfo{year}{2021}\natexlab{}.
\newblock \showarticletitle{P3: Distributed deep graph learning at scale}. In \bibinfo{booktitle}{\emph{15th $\{$USENIX$\}$ Symposium on Operating Systems Design and Implementation ($\{$OSDI$\}$ 21)}}. \bibinfo{pages}{551--568}.
\newblock


\bibitem[Gilmer et~al\mbox{.}(2017)]%
        {gilmer2017neural}
\bibfield{author}{\bibinfo{person}{Justin Gilmer}, \bibinfo{person}{Samuel~S Schoenholz}, \bibinfo{person}{Patrick~F Riley}, \bibinfo{person}{Oriol Vinyals}, {and} \bibinfo{person}{George~E Dahl}.} \bibinfo{year}{2017}\natexlab{}.
\newblock \showarticletitle{Neural message passing for quantum chemistry}. In \bibinfo{booktitle}{\emph{International conference on machine learning}}. PMLR, \bibinfo{pages}{1263--1272}.
\newblock


\bibitem[Hamilton et~al\mbox{.}(2017)]%
        {hamilton2017inductive}
\bibfield{author}{\bibinfo{person}{Will Hamilton}, \bibinfo{person}{Zhitao Ying}, {and} \bibinfo{person}{Jure Leskovec}.} \bibinfo{year}{2017}\natexlab{}.
\newblock \showarticletitle{Inductive representation learning on large graphs}.
\newblock \bibinfo{journal}{\emph{Advances in neural information processing systems}}  \bibinfo{volume}{30} (\bibinfo{year}{2017}).
\newblock


\bibitem[Jha et~al\mbox{.}(2022)]%
        {jha2022prediction}
\bibfield{author}{\bibinfo{person}{Kanchan Jha}, \bibinfo{person}{Sriparna Saha}, {and} \bibinfo{person}{Hiteshi Singh}.} \bibinfo{year}{2022}\natexlab{}.
\newblock \showarticletitle{Prediction of protein--protein interaction using graph neural networks}.
\newblock \bibinfo{journal}{\emph{Scientific Reports}} \bibinfo{volume}{12}, \bibinfo{number}{1} (\bibinfo{year}{2022}), \bibinfo{pages}{8360}.
\newblock


\bibitem[Jiang and Rumi(2021)]%
        {jiang2021communication}
\bibfield{author}{\bibinfo{person}{Peng Jiang} {and} \bibinfo{person}{Masuma~Akter Rumi}.} \bibinfo{year}{2021}\natexlab{}.
\newblock \showarticletitle{Communication-efficient sampling for distributed training of graph convolutional networks}.
\newblock \bibinfo{journal}{\emph{arXiv preprint arXiv:2101.07706}} (\bibinfo{year}{2021}).
\newblock


\bibitem[Jumper et~al\mbox{.}(2021)]%
        {jumper2021highly}
\bibfield{author}{\bibinfo{person}{John Jumper}, \bibinfo{person}{Richard Evans}, \bibinfo{person}{Alexander Pritzel}, \bibinfo{person}{Tim Green}, \bibinfo{person}{Michael Figurnov}, \bibinfo{person}{Olaf Ronneberger}, \bibinfo{person}{Kathryn Tunyasuvunakool}, \bibinfo{person}{Russ Bates}, \bibinfo{person}{Augustin {\v{Z}}{\'\i}dek}, \bibinfo{person}{Anna Potapenko}, {et~al\mbox{.}}} \bibinfo{year}{2021}\natexlab{}.
\newblock \showarticletitle{Highly accurate protein structure prediction with AlphaFold}.
\newblock \bibinfo{journal}{\emph{nature}} \bibinfo{volume}{596}, \bibinfo{number}{7873} (\bibinfo{year}{2021}), \bibinfo{pages}{583--589}.
\newblock


\bibitem[Kan et~al\mbox{.}(2022)]%
        {kan2022fbnetgen}
\bibfield{author}{\bibinfo{person}{Xuan Kan}, \bibinfo{person}{Hejie Cui}, \bibinfo{person}{Joshua Lukemire}, \bibinfo{person}{Ying Guo}, {and} \bibinfo{person}{Carl Yang}.} \bibinfo{year}{2022}\natexlab{}.
\newblock \showarticletitle{Fbnetgen: Task-aware gnn-based fmri analysis via functional brain network generation}. In \bibinfo{booktitle}{\emph{International Conference on Medical Imaging with Deep Learning}}. PMLR, \bibinfo{pages}{618--637}.
\newblock


\bibitem[Karypis and Kumar(1998)]%
        {karypis1998fast}
\bibfield{author}{\bibinfo{person}{George Karypis} {and} \bibinfo{person}{Vipin Kumar}.} \bibinfo{year}{1998}\natexlab{}.
\newblock \showarticletitle{A fast and high quality multilevel scheme for partitioning irregular graphs}.
\newblock \bibinfo{journal}{\emph{SIAM Journal on scientific Computing}} \bibinfo{volume}{20}, \bibinfo{number}{1} (\bibinfo{year}{1998}), \bibinfo{pages}{359--392}.
\newblock


\bibitem[Keahey et~al\mbox{.}(2020)]%
        {keahey2020lessons}
\bibfield{author}{\bibinfo{person}{Kate Keahey}, \bibinfo{person}{Jason Anderson}, \bibinfo{person}{Zhuo Zhen}, \bibinfo{person}{Pierre Riteau}, \bibinfo{person}{Paul Ruth}, \bibinfo{person}{Dan Stanzione}, \bibinfo{person}{Mert Cevik}, \bibinfo{person}{Jacob Colleran}, \bibinfo{person}{Haryadi~S. Gunawi}, \bibinfo{person}{Cody Hammock}, \bibinfo{person}{Joe Mambretti}, \bibinfo{person}{Alexander Barnes}, \bibinfo{person}{Fran\c{c}ois Halbach}, \bibinfo{person}{Alex Rocha}, {and} \bibinfo{person}{Joe Stubbs}.} \bibinfo{year}{2020}\natexlab{}.
\newblock \showarticletitle{Lessons Learned from the Chameleon Testbed}.
\newblock In \bibinfo{booktitle}{\emph{Proceedings of the 2020 USENIX Annual Technical Conference (USENIX ATC '20)}}. \bibinfo{publisher}{USENIX Association}.
\newblock


\bibitem[Kipf and Welling(2016)]%
        {kipf2016semi}
\bibfield{author}{\bibinfo{person}{Thomas~N Kipf} {and} \bibinfo{person}{Max Welling}.} \bibinfo{year}{2016}\natexlab{}.
\newblock \showarticletitle{Semi-supervised classification with graph convolutional networks}.
\newblock \bibinfo{journal}{\emph{arXiv preprint arXiv:1609.02907}} (\bibinfo{year}{2016}).
\newblock


\bibitem[Li et~al\mbox{.}(2022)]%
        {li2022multiphysical}
\bibfield{author}{\bibinfo{person}{Xiao-Shuang Li}, \bibinfo{person}{Xiang Liu}, \bibinfo{person}{Le Lu}, \bibinfo{person}{Xian-Sheng Hua}, \bibinfo{person}{Ying Chi}, {and} \bibinfo{person}{Kelin Xia}.} \bibinfo{year}{2022}\natexlab{}.
\newblock \showarticletitle{Multiphysical graph neural network (MP-GNN) for COVID-19 drug design}.
\newblock \bibinfo{journal}{\emph{Briefings in bioinformatics}} \bibinfo{volume}{23}, \bibinfo{number}{4} (\bibinfo{year}{2022}), \bibinfo{pages}{bbac231}.
\newblock


\bibitem[{Meta Platforms, Inc.}(2025)]%
        {meta2025q4}
\bibfield{author}{\bibinfo{person}{{Meta Platforms, Inc.}}} \bibinfo{year}{2025}\natexlab{}.
\newblock \bibinfo{title}{Meta Reports Fourth Quarter and Full Year 2024 Results}.
\newblock
\urldef\tempurl%
\url{https://investor.atmeta.com/investor-news/press-release-details/2025/Meta-Reports-Fourth-Quarter-and-Full-Year-2024-Results/default.aspx}
\showURL{%
\tempurl}
\newblock
\shownote{Accessed: 2025-04-29}.


\bibitem[{NVIDIA}(2024)]%
        {nvml_lib}
\bibfield{author}{\bibinfo{person}{{NVIDIA}}.} \bibinfo{year}{2024}\natexlab{}.
\newblock \bibinfo{title}{{NVIDIA Management Library (NVML)}}.
\newblock \bibinfo{howpublished}{\url{https://tinyurl.com/35x5pmzf}}.
\newblock
\newblock
\shownote{Accessed: 2025-04-24}.


\bibitem[{psutil}(2024)]%
        {psutil_lib}
\bibfield{author}{\bibinfo{person}{{psutil}}.} \bibinfo{year}{2024}\natexlab{}.
\newblock \bibinfo{title}{{psutil 7.0.0}}.
\newblock \bibinfo{howpublished}{\url{https://tinyurl.com/35x5pmzf}}.
\newblock
\newblock
\shownote{Accessed: 2025-04-24}.


\bibitem[R{\'e}au et~al\mbox{.}(2023)]%
        {reau2023deeprank}
\bibfield{author}{\bibinfo{person}{Manon R{\'e}au}, \bibinfo{person}{Nicolas Renaud}, \bibinfo{person}{Li~C Xue}, {and} \bibinfo{person}{Alexandre~MJJ Bonvin}.} \bibinfo{year}{2023}\natexlab{}.
\newblock \showarticletitle{DeepRank-GNN: a graph neural network framework to learn patterns in protein--protein interfaces}.
\newblock \bibinfo{journal}{\emph{Bioinformatics}} \bibinfo{volume}{39}, \bibinfo{number}{1} (\bibinfo{year}{2023}), \bibinfo{pages}{btac759}.
\newblock


\bibitem[Reiser et~al\mbox{.}(2022)]%
        {reiser2022graph}
\bibfield{author}{\bibinfo{person}{Patrick Reiser}, \bibinfo{person}{Marlen Neubert}, \bibinfo{person}{Andr{\'e} Eberhard}, \bibinfo{person}{Luca Torresi}, \bibinfo{person}{Chen Zhou}, \bibinfo{person}{Chen Shao}, \bibinfo{person}{Houssam Metni}, \bibinfo{person}{Clint van Hoesel}, \bibinfo{person}{Henrik Schopmans}, \bibinfo{person}{Timo Sommer}, {et~al\mbox{.}}} \bibinfo{year}{2022}\natexlab{}.
\newblock \showarticletitle{Graph neural networks for materials science and chemistry}.
\newblock \bibinfo{journal}{\emph{Communications Materials}} \bibinfo{volume}{3}, \bibinfo{number}{1} (\bibinfo{year}{2022}), \bibinfo{pages}{93}.
\newblock


\bibitem[Shao et~al\mbox{.}(2024)]%
        {shao2024distributed}
\bibfield{author}{\bibinfo{person}{Yingxia Shao}, \bibinfo{person}{Hongzheng Li}, \bibinfo{person}{Xizhi Gu}, \bibinfo{person}{Hongbo Yin}, \bibinfo{person}{Yawen Li}, \bibinfo{person}{Xupeng Miao}, \bibinfo{person}{Wentao Zhang}, \bibinfo{person}{Bin Cui}, {and} \bibinfo{person}{Lei Chen}.} \bibinfo{year}{2024}\natexlab{}.
\newblock \showarticletitle{Distributed graph neural network training: A survey}.
\newblock \bibinfo{journal}{\emph{Comput. Surveys}} \bibinfo{volume}{56}, \bibinfo{number}{8} (\bibinfo{year}{2024}), \bibinfo{pages}{1--39}.
\newblock


\bibitem[Shlomi et~al\mbox{.}(2020)]%
        {shlomi2020graph}
\bibfield{author}{\bibinfo{person}{Jonathan Shlomi}, \bibinfo{person}{Peter Battaglia}, {and} \bibinfo{person}{Jean-Roch Vlimant}.} \bibinfo{year}{2020}\natexlab{}.
\newblock \showarticletitle{Graph neural networks in particle physics}.
\newblock \bibinfo{journal}{\emph{Machine Learning: Science and Technology}} \bibinfo{volume}{2}, \bibinfo{number}{2} (\bibinfo{year}{2020}), \bibinfo{pages}{021001}.
\newblock


\bibitem[Thorpe et~al\mbox{.}(2021)]%
        {thorpe2021dorylus}
\bibfield{author}{\bibinfo{person}{John Thorpe}, \bibinfo{person}{Yifan Qiao}, \bibinfo{person}{Jonathan Eyolfson}, \bibinfo{person}{Shen Teng}, \bibinfo{person}{Guanzhou Hu}, \bibinfo{person}{Zhihao Jia}, \bibinfo{person}{Jinliang Wei}, \bibinfo{person}{Keval Vora}, \bibinfo{person}{Ravi Netravali}, \bibinfo{person}{Miryung Kim}, {et~al\mbox{.}}} \bibinfo{year}{2021}\natexlab{}.
\newblock \showarticletitle{Dorylus: Affordable, scalable, and accurate $\{$GNN$\}$ training with distributed $\{$CPU$\}$ servers and serverless threads}. In \bibinfo{booktitle}{\emph{15th USENIX Symposium on Operating Systems Design and Implementation (OSDI 21)}}. \bibinfo{pages}{495--514}.
\newblock


\bibitem[Wan et~al\mbox{.}(2023)]%
        {wan2023adaptive}
\bibfield{author}{\bibinfo{person}{Borui Wan}, \bibinfo{person}{Juntao Zhao}, {and} \bibinfo{person}{Chuan Wu}.} \bibinfo{year}{2023}\natexlab{}.
\newblock \showarticletitle{Adaptive message quantization and parallelization for distributed full-graph gnn training}.
\newblock \bibinfo{journal}{\emph{Proceedings of Machine Learning and Systems}}  \bibinfo{volume}{5} (\bibinfo{year}{2023}), \bibinfo{pages}{203--218}.
\newblock


\bibitem[Wan et~al\mbox{.}(2022)]%
        {wan2022dgs}
\bibfield{author}{\bibinfo{person}{Xinchen Wan}, \bibinfo{person}{Kai Chen}, {and} \bibinfo{person}{Yiming Zhang}.} \bibinfo{year}{2022}\natexlab{}.
\newblock \showarticletitle{Dgs: Communication-efficient graph sampling for distributed gnn training}. In \bibinfo{booktitle}{\emph{2022 IEEE 30th International Conference on Network Protocols (ICNP)}}. IEEE, \bibinfo{pages}{1--11}.
\newblock


\bibitem[Wang et~al\mbox{.}(2024)]%
        {wang2024sc}
\bibfield{author}{\bibinfo{person}{Jihe Wang}, \bibinfo{person}{Ying Wu}, {and} \bibinfo{person}{Danghui Wang}.} \bibinfo{year}{2024}\natexlab{}.
\newblock \showarticletitle{SC-GNN: A Communication-Efficient Semantic Compression for Distributed Training of GNNs}. In \bibinfo{booktitle}{\emph{Proceedings of the 61st ACM/IEEE Design Automation Conference}}. \bibinfo{pages}{1--6}.
\newblock


\bibitem[Wang et~al\mbox{.}(2019)]%
        {wang2019deep}
\bibfield{author}{\bibinfo{person}{Minjie Wang}, \bibinfo{person}{Da Zheng}, \bibinfo{person}{Zihao Ye}, \bibinfo{person}{Quan Gan}, \bibinfo{person}{Mufei Li}, \bibinfo{person}{Xiang Song}, \bibinfo{person}{Jinjing Zhou}, \bibinfo{person}{Chao Ma}, \bibinfo{person}{Lingfan Yu}, \bibinfo{person}{Yu Gai}, {et~al\mbox{.}}} \bibinfo{year}{2019}\natexlab{}.
\newblock \showarticletitle{Deep graph library: A graph-centric, highly-performant package for graph neural networks}.
\newblock \bibinfo{journal}{\emph{arXiv preprint arXiv:1909.01315}} (\bibinfo{year}{2019}).
\newblock


\bibitem[Xiong et~al\mbox{.}(2021)]%
        {xiong2021graph}
\bibfield{author}{\bibinfo{person}{Jiacheng Xiong}, \bibinfo{person}{Zhaoping Xiong}, \bibinfo{person}{Kaixian Chen}, \bibinfo{person}{Hualiang Jiang}, {and} \bibinfo{person}{Mingyue Zheng}.} \bibinfo{year}{2021}\natexlab{}.
\newblock \showarticletitle{Graph neural networks for automated de novo drug design}.
\newblock \bibinfo{journal}{\emph{Drug discovery today}} \bibinfo{volume}{26}, \bibinfo{number}{6} (\bibinfo{year}{2021}), \bibinfo{pages}{1382--1393}.
\newblock


\bibitem[Zeng et~al\mbox{.}(2019)]%
        {zeng2019graphsaint}
\bibfield{author}{\bibinfo{person}{Hanqing Zeng}, \bibinfo{person}{Hongkuan Zhou}, \bibinfo{person}{Ajitesh Srivastava}, \bibinfo{person}{Rajgopal Kannan}, {and} \bibinfo{person}{Viktor Prasanna}.} \bibinfo{year}{2019}\natexlab{}.
\newblock \showarticletitle{Graphsaint: Graph sampling based inductive learning method}.
\newblock \bibinfo{journal}{\emph{arXiv preprint arXiv:1907.04931}} (\bibinfo{year}{2019}).
\newblock


\bibitem[Zhang et~al\mbox{.}(2023a)]%
        {zhang2023boosting}
\bibfield{author}{\bibinfo{person}{Meng Zhang}, \bibinfo{person}{Qinghao Hu}, \bibinfo{person}{Peng Sun}, \bibinfo{person}{Yonggang Wen}, {and} \bibinfo{person}{Tianwei Zhang}.} \bibinfo{year}{2023}\natexlab{a}.
\newblock \showarticletitle{Boosting distributed full-graph gnn training with asynchronous one-bit communication}.
\newblock \bibinfo{journal}{\emph{arXiv preprint arXiv:2303.01277}} (\bibinfo{year}{2023}).
\newblock


\bibitem[Zhang et~al\mbox{.}(2023b)]%
        {zhang2023two}
\bibfield{author}{\bibinfo{person}{Zhe Zhang}, \bibinfo{person}{Ziyue Luo}, {and} \bibinfo{person}{Chuan Wu}.} \bibinfo{year}{2023}\natexlab{b}.
\newblock \showarticletitle{Two-level graph caching for expediting distributed gnn training}. In \bibinfo{booktitle}{\emph{IEEE INFOCOM 2023-IEEE Conference on Computer Communications}}. IEEE, \bibinfo{pages}{1--10}.
\newblock


\bibitem[Zheng et~al\mbox{.}(2020)]%
        {zheng2020distdgl}
\bibfield{author}{\bibinfo{person}{Da Zheng}, \bibinfo{person}{Chao Ma}, \bibinfo{person}{Minjie Wang}, \bibinfo{person}{Jinjing Zhou}, \bibinfo{person}{Qidong Su}, \bibinfo{person}{Xiang Song}, \bibinfo{person}{Quan Gan}, \bibinfo{person}{Zheng Zhang}, {and} \bibinfo{person}{George Karypis}.} \bibinfo{year}{2020}\natexlab{}.
\newblock \showarticletitle{DistDGL: Distributed graph neural network training for billion-scale graphs}. In \bibinfo{booktitle}{\emph{2020 IEEE/ACM 10th Workshop on Irregular Applications: Architectures and Algorithms (IA3)}}. IEEE, \bibinfo{pages}{36--44}.
\newblock


\bibitem[Zou et~al\mbox{.}(2019)]%
        {zou2019layer}
\bibfield{author}{\bibinfo{person}{Difan Zou}, \bibinfo{person}{Ziniu Hu}, \bibinfo{person}{Yewen Wang}, \bibinfo{person}{Song Jiang}, \bibinfo{person}{Yizhou Sun}, {and} \bibinfo{person}{Quanquan Gu}.} \bibinfo{year}{2019}\natexlab{}.
\newblock \showarticletitle{Layer-dependent importance sampling for training deep and large graph convolutional networks}.
\newblock \bibinfo{journal}{\emph{Advances in neural information processing systems}}  \bibinfo{volume}{32} (\bibinfo{year}{2019}).
\newblock


\end{thebibliography}
\end{document}